\documentclass[preprint,12pt,review]{elsarticle}

\usepackage{amsmath}
\usepackage{amssymb}
\usepackage{mathtools}
\usepackage{amsthm}
\usepackage{graphicx}
\usepackage{subfigure}
\usepackage{diagbox}
\usepackage[noend]{algorithmic}
\usepackage{algorithm}
\usepackage{hyperref}
\usepackage{enumitem}
\usepackage{changes}
\usepackage{xcolor, soul} 
\usepackage{comment}

\theoremstyle{plain}
\newtheorem{theorem}{Theorem}[section]
\newtheorem{proposition}[theorem]{Proposition}

\theoremstyle{definition}
\newtheorem{definition}[theorem]{Definition}

\theoremstyle{remark}
\newtheorem{remark}[theorem]{Remark}

\DeclareMathOperator*{\argmax}{arg\,max}

\journal{Pattern Recognition}

\begin{document}

\begin{frontmatter}

\author[label1]{Junyoung Byun}
\author[label2]{Yujin Choi}
\author[label2]{Jaewook Lee\corref{cor1}}
\cortext[cor1]{Corresponding author at: Seoul National University, 1, Gwanak-ro, Seoul, 08826, Republic of Korea \\ E-mail addresses: junyoungb@cau.ac.kr (J. Byun), uznhigh@snu.ac.kr (Y. Choi), jaewook@snu.ac.kr (J. Lee)}
\affiliation[label1]{organization={Chung-Ang University},%Department and Organization
            addressline={84, Heukseok-ro}, 
            city={Seoul},
            postcode={06974}, 
            %state={},
            country={Republic of Korea}}

\affiliation[label2]{organization={Seoul National University},%Department and Organization
            addressline={1, Gwanak-ro}, 
            city={Seoul},
            postcode={08826}, 
            %state={},
            country={Republic of Korea}}

\title{Improving the Utility of Differentially Private Clustering through Dynamical Processing}

\begin{abstract}

This study aims to alleviate the trade-off between utility and privacy of differentially private clustering. Existing works focus on simple methods, which show poor performance for non-convex clusters. \textcolor{black}{To fit complex cluster distributions}, we propose sophisticated dynamical processing inspired by Morse theory, with which we hierarchically connect the Gaussian sub-clusters obtained through existing methods. 
Our theoretical results imply that the proposed dynamical processing introduces little to no additional privacy loss.
%In addition, the proposed method is inductive and can achieve any desired number of clusters. 
Experiments show that our framework can improve the clustering performance of existing methods at the same privacy level.
\end{abstract}


% \begin{highlights}
% \item We introduce a comprehensive framework for differentially private clustering, which can enhance the utility of a wide range of existing methods. In the case of mixture-of-Gaussians-based methods, the proposed framework introduces no additional privacy loss. For k-means-based methods, it incurs a little additional privacy to estimate the mixture of Gaussians density function.
% \item We propose a dynamical processing associated with the mixture of Gaussians density function, leveraging Morse theory, and prove its inductive property, obviating the need for retraining when clustering new data. 
% \item Building upon the proposed dynamical processing, we suggest a hierarchical procedure for the initial sub-clusters, which is shown to be differential-privacy-friendly thereby avoiding any additional privacy loss. Through the hierarchical merging of sub-clusters, the proposed method can identify more complex clusters compared to existing differentially private clustering methods while maintaining the same privacy level.
% \end{highlights}

\begin{keyword}


Clustering, Differential privacy, Dynamical Processing, Morse Theory
\end{keyword}

\end{frontmatter}


\section{Introduction}
\label{introduction}

\textcolor{black}{Machine learning algorithms can inadvertently expose sensitive information about individuals used for training. Differential privacy (DP) offers a mathematical guarantee to prevent the leakage of personal information from algorithm outputs, although it often comes at the cost of reduced algorithm performance. Consequently, research in differentially private machine learning focuses on maintaining algorithm performance while ensuring privacy.}% protection.

Clustering is widely used in applications like recommendation systems, marketing, and fraud detection, where protecting customer information is crucial \textcolor{black}{\citep{ribero2022federating,chatterjee2024digital}}. However, differentially private clustering has received less attention compared to supervised learning. Existing studies primarily address simple methods such as k-means clustering \textcolor{black}{\citep{epasto2024k,li2024gapbas}} and mixtures of Gaussians (MoGs) \textcolor{black}{\citep{park2017dp, li2022privacy,afzali2024mixtures}}, which struggle to represent complex, nonconvex cluster structures.

Inspired by Morse theory from differential topology, Morse clustering remains an active area of research \textcolor{black}{\citep{strazzeri2022possibility, pandey2022morse}}. \cite{lee2009dynamic} applied Morse theory to enhance the performance of non-private support vector-based clustering algorithms. 
However, applying DP to their method is challenging because calculating the kernel function during inference requires support vectors from the training data.

This study addresses these limitations by applying Morse theory to enhance differentially private clustering algorithms. Our method leverages MoG models, which do not require training data once the density function is estimated.
From the MoG density function obtained through differentially private clustering, we build an associated dynamical system and construct a weighted graph to connect Gaussian sub-clusters, effectively representing complex clusters through the merging of sub-clusters.

Theoretical results demonstrate that the proposed dynamical processing is DP-friendly, introducing minimal additional privacy loss beyond existing methods. Furthermore, our method is inductive and capable of achieving any desired number of clusters.

\section{Related Works}

%\cite{chang2021locally} proposed a method for computing differentially private coresets for k-means clustering. 
Recent research has focused on improving the sample-and-aggregate framework. \cite{cohen2021differentially} introduced an enhanced framework for well-separated datasets, offering better DP guarantees and sample complexity bounds. \textcolor{black}{\cite{kamath2022new}} applied a sample-and-aggregate framework to MoG models.
%Other differentially private clustering algorithms have also been explored. 
\cite{ni2018dp} and \textcolor{black}{\cite{wu2024density}} investigated differentially private DBSCAN algorithms. \textcolor{black}{Recently, \cite{li2024gapbas} proposed a genetic algorithm-based privacy budget allocation strategy in differentially private k-means algorithm.}

\label{related}

\section{Differential Privacy}
\label{dp}

%DP \citep{dwork2014algorithmic} is a definition of privacy that offers an upper bound on the changes in a mechanism's output due to any change in the dataset.

DP \citep{dwork2014algorithmic} defines privacy by limiting how much a mechanism's output can change with alterations in the dataset.

\begin{definition}\label{def:dp}
(Differential privacy)
A randomized mechanism $\mathcal{M}$ is \textit{${(\epsilon,\delta)}$-differentially private} if for any set of possible outputs $\mathcal{S} \subseteq Range(\mathcal{M})$ and two neighboring datasets $D,D'\in\mathcal{D}$ which differ in exactly one data sample:   
\begin{equation}
Pr[\mathcal{M}(D) \in \mathcal{S}] \leq e^\epsilon Pr[\mathcal{M}(D') \in \mathcal{S}]+\delta.    
\end{equation}
When $\delta = 0$, the mechanism $\mathcal{M}$ is called ${\epsilon}$-differentially private.
\end{definition}

In this definition, $\epsilon$ is called the \textit{privacy budget}, indicating that privacy degrades as $\epsilon$ increases. A well-known result of DP is composition, which states that privacy degrades with multiple data accesses. While the most widely recognized composition method is advanced composition \citep{dwork2014algorithmic}, this study utilizes a variant of DP called zero-concentrated DP (zCDP) \citep{bun2016concentrated}, which provides a tighter analysis of privacy loss. Below we introduce some properties of zCDP relevant to this paper.

\begin{remark}\label{remarks1}
    (Proposition 1.3 \cite{bun2016concentrated}) If $\mathcal{M}$ satisfies $\rho$-zCDP, then $\mathcal{M}$ satisfies $(\rho+2\sqrt{\rho\log{(1/\delta)}}, \delta)$-DP for any $\delta>0$.
\end{remark}

\begin{remark}\label{remarks2}
    (Proposition 1.4 \cite{bun2016concentrated}) If $\mathcal{M}$ satisfies $\epsilon$-DP, then $\mathcal{M}$ satisfies $\frac{1}{2}\epsilon^2$-zCDP.
\end{remark}

\begin{remark}\label{remarks3}
    (Proposition 1.6 \cite{bun2016concentrated}) Let function $f : \mathcal{X}^n\rightarrow\mathbb{R}$ has a sensitivity $\Delta$. Then on input $x$, releasing a sample from $\mathcal{N}(f(x), \sigma^2)$ satisfies $(\Delta^2/2\sigma^2)$-zCDP.
\end{remark}

\begin{remark}\label{remarks4}
    (Lemma 1.8 \cite{bun2016concentrated}) Let randomized mechanisms $\mathcal{M}_1$ and $\mathcal{M}_2$ satisfy $\rho_1$-zCDP and $\rho_2$-zCDP, respectively. Then the mapping $\mathcal{M}_{1,2}=(\mathcal{M}_1, \mathcal{M}_2)$ satisfies $(\rho_1+\rho_2)$-zCDP.
\end{remark}

Two important properties of DP, \textcolor{black}{post-processing and parallel composition,} allow for controlling privacy loss. Post-processing states that any processing that does not access the data does not increase the privacy loss.

\begin{remark} %% property가 맞을 것 같은데, 일단 양식에는 정의가 안 되어 있어서 remark로 뺌
(Post-Processing) Let the randomized mechanism $\mathcal{M}$ be $(\epsilon, \delta)$-differentially private. Then, for any data-independent randomized mapping $f$, $f \circ \mathcal{M}$ is also $(\epsilon, \delta)$-differentially private.
\end{remark}

When the data is partitioned, we can exploit parallel composition.
\begin{remark}
(Parallel Composition) Let the dataset $D$ be partitioned by disjoint subset $D_i$ for $i = 1,2,\dots, n$ and let $\mathcal{M}_i$ is $\epsilon_i$ - differentially private mechanism which takes $D_i$ as input. Then, releasing all of the results $\mathcal{M}_1, \mathcal{M}_2, \dots, \mathcal{M}_n$ is $\max_i\epsilon_i$-differentially private. 
\end{remark}

\section{Morse Theory from a Dynamical System Perspective}\label{morse}

\subsection{Morse theory}

Let $\{\boldsymbol{x}_i\}\subset \mathcal{X}$ be a given dataset of $N$
points, with $\mathcal{X}:=\mathbb{R}^D$, the data space. Given a smooth real-valued function $\mathit{f}:\mathcal{X}\to \mathbb{R}$ mapping each point to its height, the inverse image of a point $a\in \mathbb{R}$, called a level set, can be decomposed into several separate connected components $C_i,\; i=1,...,m$, i.e.,
\begin{equation}
\begin{aligned}
\mathcal{X}_a:=&\mathit{f}^{-1}(-\infty, a]
=\{\boldsymbol{x}\in \mathcal{X}:\mathit{f}(\boldsymbol{x})\leq a\}=C_1
\cup\cdots\cup C_m\label{eq:4}
\end{aligned}
\end{equation}

A state vector $\boldsymbol{x}$ satisfying $\nabla \mathit{f}(\boldsymbol{x})=0$ is called an {\em equilibrium vector} of $\mathit{f}$. An equilibrium vector $\boldsymbol{x}$ is {\em hyperbolic} if the Hessian of $f$ at $\boldsymbol{x}$ restricted to the tangent space to $\mathcal{X}$ at $\boldsymbol{x}$, denoted by $H_\mathit{f}(\boldsymbol{x})$, has no zero eigenvalues. A hyperbolic equilibrium vector $\boldsymbol{x}$ is called an {\em index-k equilibrium vector} if $H_\mathit{f}(\boldsymbol{x})$ has $k$ negative eigenvalues. The index corresponds to the dimension of the subspace consisting of directions in which $f$ decreases. 
We denote an index-$k$ equilibrium vector by $\boldsymbol{x}^k$.

A function $f$ is a {\em Morse function} if all its equilibrium vectors are hyperbolic. Additionally, $f$ is {\em separating} if distinct equilibrium vectors have distinct functional values.
A basic result of Morse theory \citep{hirsch2012differential, palais2000generalized} is that the class of the Morse functions forms an open, dense subset of all the smooth functions in the $C^{2}$-topology, in other words, almost all smooth functions are Morse functions. Therefore, in this paper, it is assumed that the function considered is Morse, which is ``generic''.

\begin{figure}[t]
\centering
\subfigure[]{%
\includegraphics[width=0.48\linewidth]{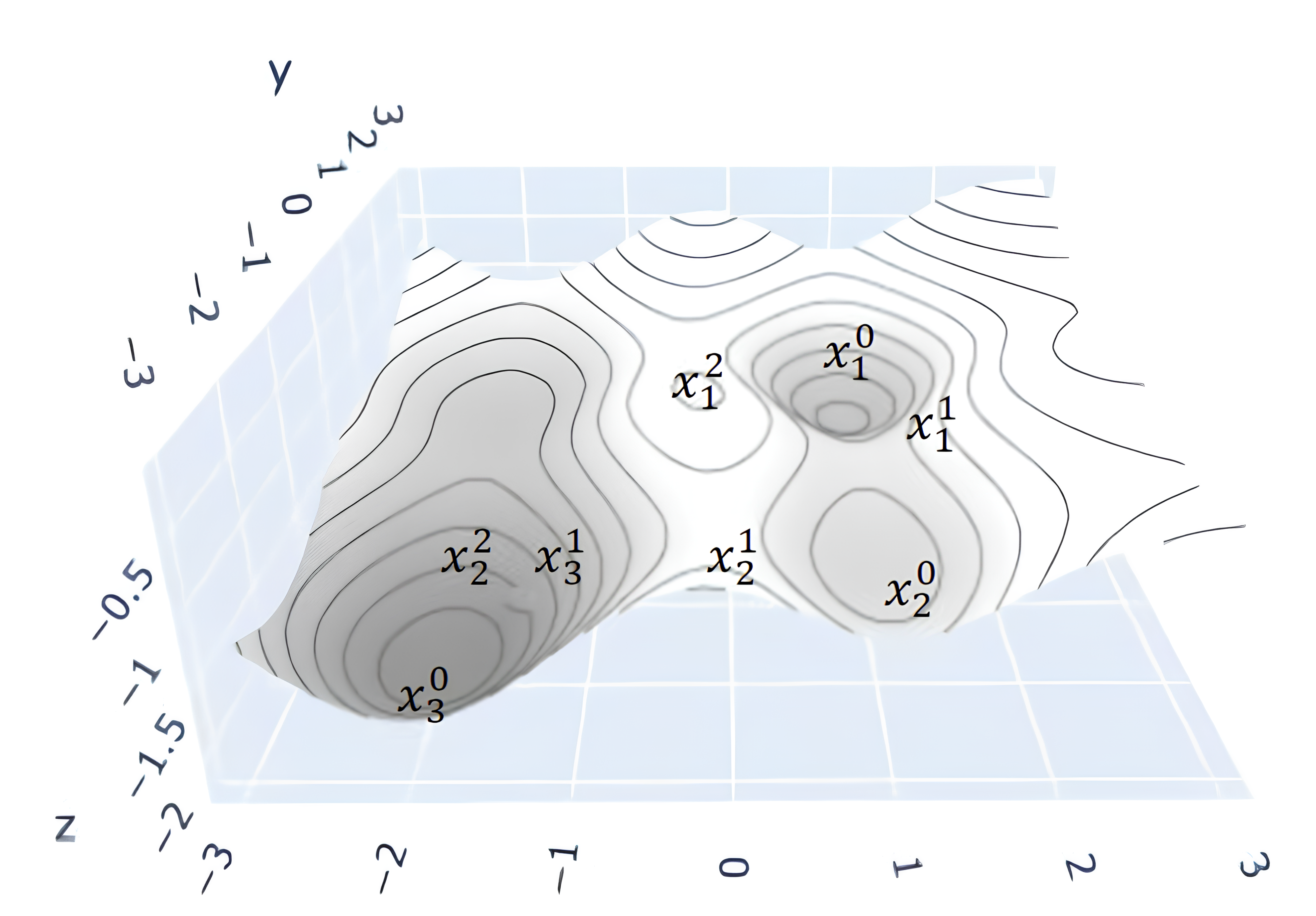}
\label{fig:MoG_morse}}
\subfigure[]{%
\includegraphics[width=0.48\linewidth]{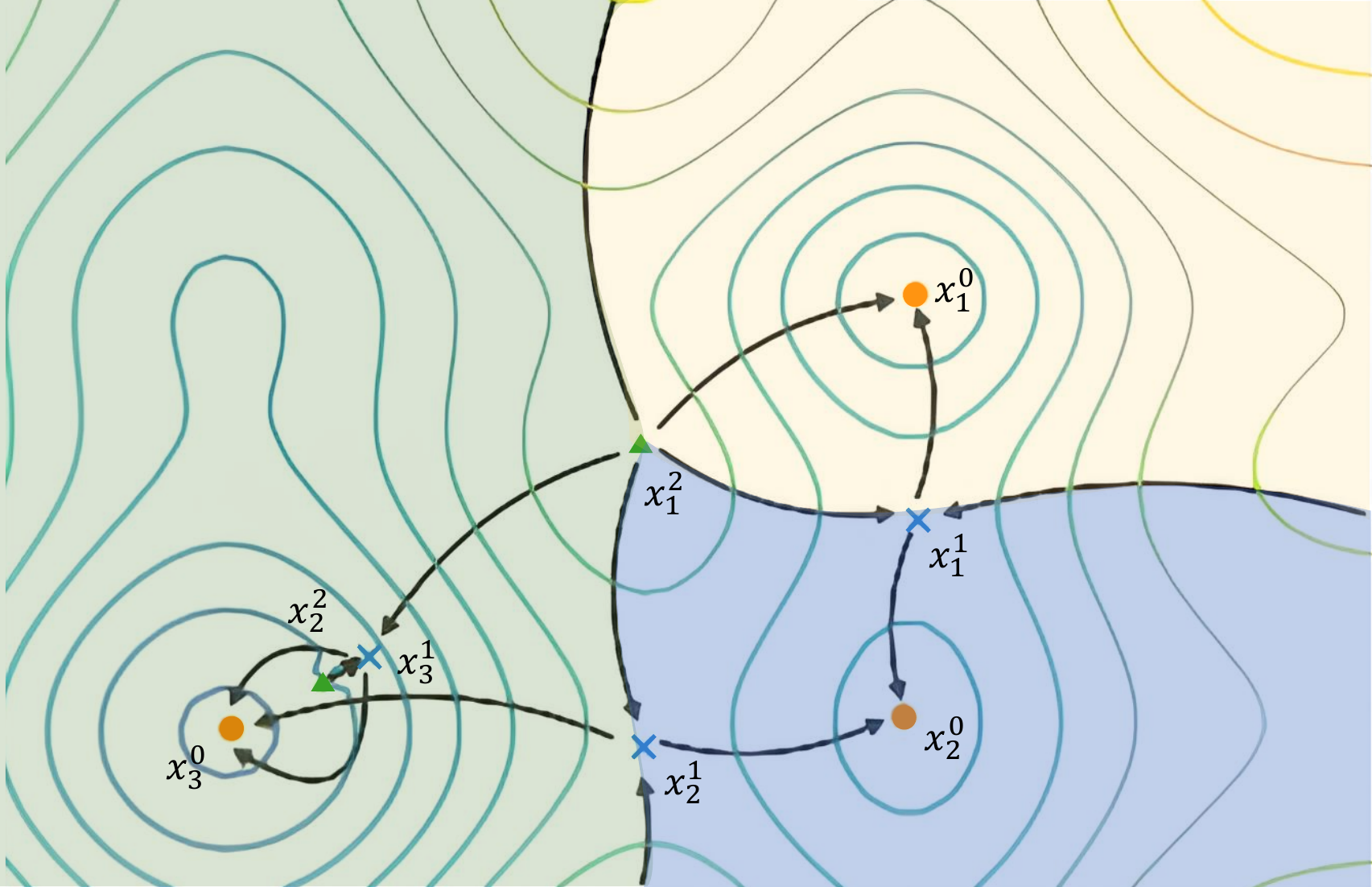}
\label{fig:morse_ex}}
\caption{Illustration of Morse theory. \textcolor{black}{(a) A three-dimensional surface plot of a toy example.} (b) A level curve of (a). In (b), `o's are stable equilibrium points, and `x' refers to the index-1 equilibrium point.  The triangle marker refers to the index-two equilibrium point. Note that
$\boldsymbol{x}_3^1$ is an index-1 equilibrium vector, but not a TEV.}
\label{fig:morse}
\end{figure}%뾰족한 부분은 index 2 고, 그 뒤에 쏙 들어간 부분이 index 1 우리가 말하는 부분이다
Morse theory addresses when the topology of $\#|\mathcal{X}_a|$ changes as $a$ varies as follows: ($\#|A|$ denotes the number of connected components of a set $A$. See \citep{hirsch2012differential,  palais2000generalized} for more details.)
\begin{itemize}
\item $\#|\mathcal{X}_a|$ increases by one, i.e. $\#|\mathcal{X}_{a+\varepsilon}|=\#|\mathcal{X}_{a-\varepsilon}|+1$ for a sufficiently small $\varepsilon>0$, if and only if, $a\in \{\mathit{f}(\boldsymbol{x}_1^0),...,\mathit{f}(\boldsymbol{x}_s^0)\}$.
\item $\#|\mathcal{X}_a|$ decreases by one, i.e. $\#|\mathcal{X}_{a+\varepsilon}|=\#|\mathcal{X}_{a-\varepsilon}|-1$ for a sufficiently small $\varepsilon>0$, if and only if, $a\in \{\mathit{f}(\boldsymbol{x}_1^1),...,\mathit{f}(\boldsymbol{x}_t^1)\}$ and also the following Morse relation holds.
\begin{equation}
\begin{aligned}
H_{0}(\mathcal{X}_{a-\varepsilon})&\cong H_{0}(\mathcal{X}_{a+\varepsilon})\oplus \mathbb{R},\\
H_{q}(\mathcal{X}_{a-\varepsilon})&\cong H_{q}(\mathcal{X}_{a+\varepsilon}),~~\mbox{for}~~q>0 \label{eq:524}
\end{aligned}
\end{equation}
where $H_{q}(\cdot)$ is the $q$-th homology space and $\cong$ implies homotopy equivalent. Such $\boldsymbol{x}_i^1$ are called {\em transition equilibrium vectors} (TEVs) (e.g., in Figure \ref{fig:morse}, $\boldsymbol{x}_1^1$, $\boldsymbol{x}_2^1$ are TEVs, but $\boldsymbol{x}_3^1$ is not, \textcolor{black}{as $\#|\mathcal{X}|$ remains constant despite being an index-1 equilibrium point}).
\item $\#|\mathcal{X}_a|$  remains constant,  i.e., $\#|\mathcal{X}_{a+\varepsilon}|=\#|\mathcal{X}_{a-\varepsilon}|$ for a sufficiently small $\varepsilon>0$, if and only if, $a$ passes the value $\mathit{f}(\boldsymbol{x}^k)$ of an index-k equilibrium vector $\boldsymbol{x}^k$ with $k>1$.
    \end{itemize}

\subsection{Dynamical system perspective}

Morse theory is not directly applicable due to the difficulty of computing the $q$-th homology space $H_{q}(\cdot)$, for example. The generalized gradient vector fields can help compute it \citep{rot2014morse}. %introduced by \cite{smale1961gradient} 
Associated with the Morse function $\mathit{f}$, we can build the following generalized gradient system:
\begin{eqnarray}
\frac{d{\boldsymbol{x}}}{dt}=-{\mathrm{grad}}_R \mathit{f}({\boldsymbol{x}})\equiv
-R({\boldsymbol{x}})^{-1}\nabla \mathit{f}({\boldsymbol{x}}).\label{eq:grad}
\end{eqnarray}
where $R(\cdot)$ is a {\em Riemannian metric} on $\mathcal{X}$ (i.e. $R({\boldsymbol{x}})$ is a positive definite symmetric matrix for all ${\boldsymbol{x}}\in \mathcal{X}$. The existence of a unique solution (or trajectory) ${\boldsymbol{x}}(\cdot):\mathbb{R}\rightarrow \mathcal{X}$ for each initial
condition ${\boldsymbol{x}}(0)$ is guaranteed by the smoothness of $f$ \citep{guckenheimer2013nonlinear, khalil2002nonlinear} (i.e. $\mathit{f}$ is twice differentiable). Without loss of generality, we assume that the trajectory ${\boldsymbol{x}}(\cdot)$ is defined for all $t\in
\mathbb{R}$ for any initial condition ${\boldsymbol{x}}(0)$, achievable under a suitable re-parametrization \citep{guckenheimer2013nonlinear}.

A hyperbolic equilibrium vector is called an (asymptotically) \textit{stable} equilibrium vector, denoted by $\boldsymbol{x}^0$, if all eigenvalues of its corresponding Jacobian are positive, and an \textit{unstable} equilibrium vector (or a \textit{repellor}), denoted by $\boldsymbol{x}^D$, if all eigenvalues of its corresponding Jacobian are negative. A basic result is that every local minimum of the Morse function $\mathit{f}$ corresponds to an (asymptotically) stable equilibrium vector (SEV) of the system \eqref{eq:grad}.

The (practical) {\em basin cell} of attraction of a SEV $\boldsymbol{x}^0$ is the closure of an open and connected stable manifold, defined by
\begin{eqnarray}
\mathfrak{B}(\boldsymbol{x}^0):=\textrm{cl}(\{{\boldsymbol{x}}(0)\in \mathcal{X}:
\lim_{t\to\infty}{\boldsymbol{x}}(t)=\boldsymbol{x}^0\}).
\end{eqnarray}
A basin cell groups similar data points through the system (\ref{eq:grad}).
Two SEVs, $\boldsymbol{x}^0_i$ and $\boldsymbol{x}^0_j$, are {\em adjacent} if there
exists an index-1 equilibrium vector $\boldsymbol{x}^1_{ij}\in \mathfrak{B}(\boldsymbol{x}^0_i)\cap
\mathfrak{B}(\boldsymbol{x}^0_j)$. It can be shown \citep{lee2004dynamical} that such an index-1 equilibrium vector is in fact a TEV between $\boldsymbol{x}^0_a$ and $\boldsymbol{x}^0_b$ that satisfies the Morse relation \eqref{eq:524} and can be computed by using the system \eqref{eq:grad}.

\section{Proposed Method}

\begin{figure}[t]
\centering
\subfigure[]{%
\includegraphics[width=0.485\linewidth]{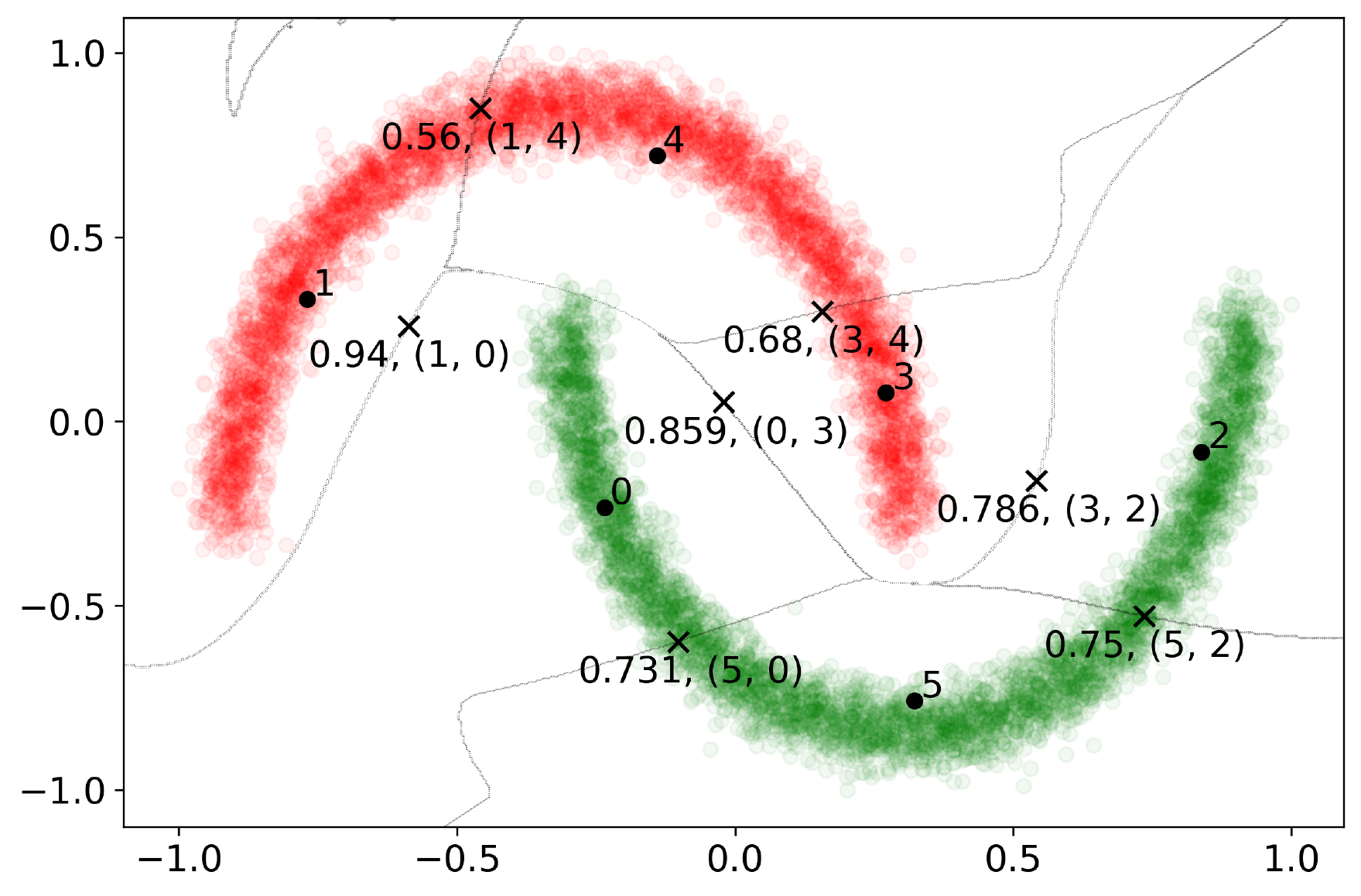}
\label{fig:after-con}}
\subfigure[]{%
\includegraphics[width=0.485\linewidth]{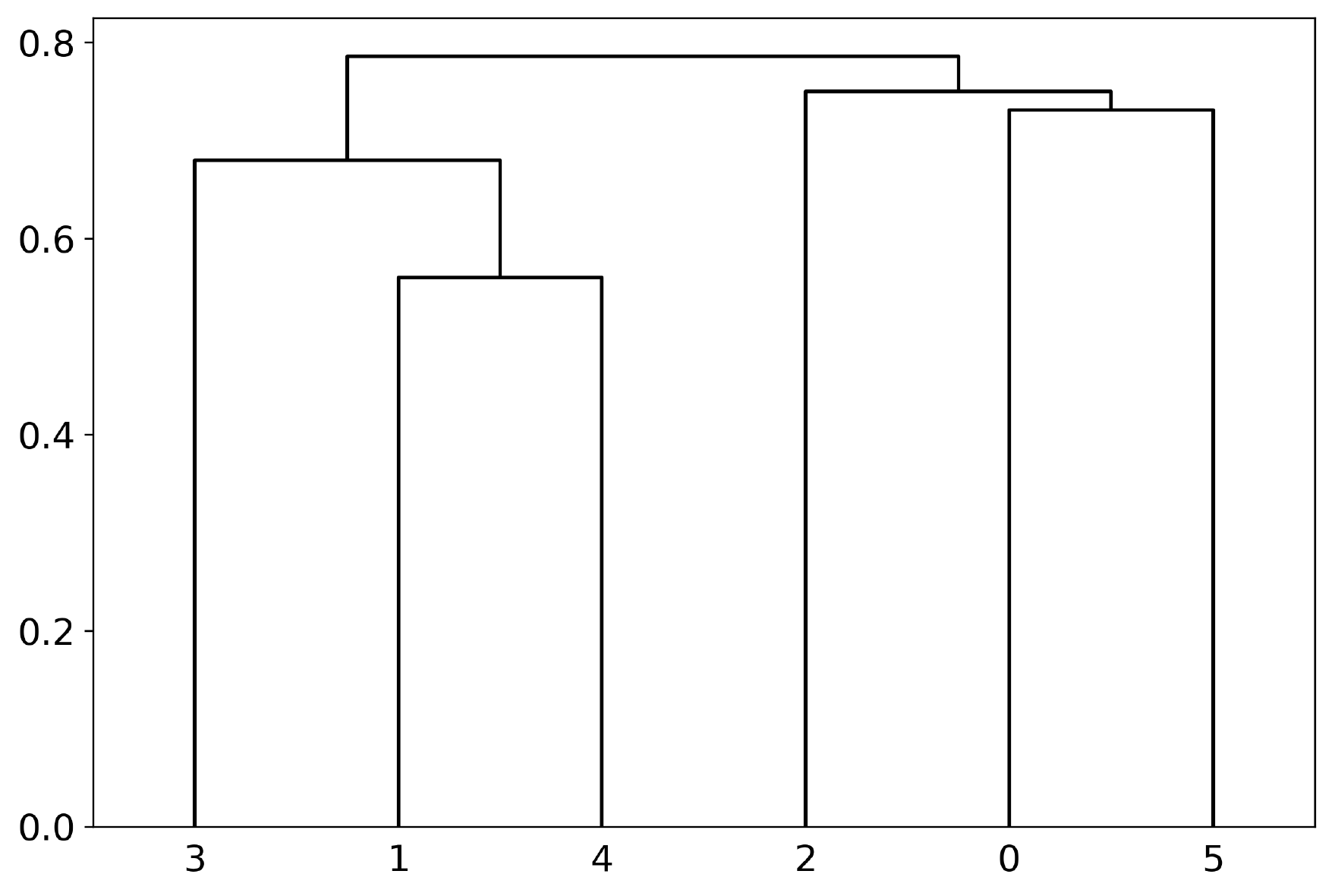}
\label{fig:dendrogram}}
\caption{Description of the proposed dynamical processing. After obtaining an MoG density from textbf{DPClustering}, (a) TEVs ('x') between adjacent centers (two numbers in parentheses) are identified, and the weight between two centers is calculated as the density (number at the bottom left) of the corresponding TEV. (b) A dendrogram is drawn based on the weighted graph. 
The x-axis denotes the index of centers, and the y-axis denotes the weight between the centers.
% Description of the proposed dynamical processing. After obtaining an MoG density from differentially private clustering, (a) TEVs between adjacent centers are identified, and the weight between two centers is calculated as the density of the corresponding TEV. TEVs are plotted as 'x', with the density indicated by a number at the bottom left of each TEV. The two numbers in parentheses indicate the indices of the centers each TEV connects. (b) A dendrogram is drawn based on the weighted graph constructed in (a) (ARI=1 when K=2). 
% The x-axis denotes the index of centers, and the y-axis denotes the weight between the centers, serving as the distance measure.
}
\label{fig:twomoon}
\end{figure}

\begin{algorithm}[h]
  \caption{Proposed Method}
  \label{alg:proposed}
\begin{algorithmic}[1]
  \STATE \textcolor{black}{{\bfseries Input:} Input data $\{\boldsymbol{x}_n\}_{n=1}^N\in [-1,1]^{N\times D}$, privacy budget ($\epsilon, \delta$), number of initial clusters $K_0$, number of final clusters $K$, number of iterations $\tau_1, \tau_2$, line search parameter $m$}
  \STATE \textcolor{black}{{\bfseries Output:} Final clusters $\{C\}_{k=1}^K$}
  
  \STATE \textcolor{black}{$\{\pi_k, \boldsymbol{\mu}_k, \boldsymbol{\Sigma}_k\}_{k=1}^{K_0} \leftarrow$ \textbf{DPClustering}(Input data, $\epsilon$, $\delta$, $K_0$, $\tau_1$)}
  \STATE \textcolor{black}{Compute $f$ from $\{\pi_k, \boldsymbol{\mu}_k, \boldsymbol{\Sigma}_k\}_{k=1}^{K_0}$}
  \STATE \textcolor{black}{$G=(V,E) \leftarrow$ \textbf{TEVGraph}($\{\boldsymbol{\mu}_k\}_{k=1}^{K_0}$, $f$, $m$, $\tau_2$)}
  \STATE \textcolor{black}{$\{C\}_{k=1}^K \leftarrow$ \textbf{MergeCluster}($G$, $K$)}
\end{algorithmic}
\end{algorithm}

The proposed method consists of three steps as summarized in Algorithm \ref{alg:proposed}.
First, differentially private parameters of the density function $p$, assumed to be an MoG with $K_0$ clusters, are estimated from any existing method (line 3). We consider $-\ln p$ as the Morse function $f$. The centers of the MoG and their corresponding clusters effectively approximate SEVs and the associated basin cells, respectively. Subsequently, TEVs of $f$ are identified according to \eqref{eq:grad}, and a graph $G$ is constructed (line 5). Finally, adjacent centers are connected until the desired number of clusters $K$ is achieved (line 6).  Figure \ref{fig:twomoon} provides a visualization of the second and third steps.

\subsection{Differentially private clustering}

An MoG function can be represented as:
\begin{equation}
p(\boldsymbol{x})=\sum_{k=1}^K \pi_k(2\pi)^{-\frac{D}{2}}|\boldsymbol{\Sigma}_k|^{-\frac{1}{2}}
e^{-\frac{1}{2}(\boldsymbol{x}-\boldsymbol{\mu}_k)^T\boldsymbol{\Sigma}_k^{-1}
(\boldsymbol{x}-\boldsymbol{\mu}_k)}, \label{eq:mog2}
\end{equation}
where $\boldsymbol{\mu}_k\in\mathbb{R}^D, \boldsymbol{\Sigma}_k\in\mathbb{R}^{D\times D}$ are the mean and covariance of each normal distribution, and $\pi_k$ is the probability of belonging to the k-th cluster.
The associated MoG gradient system can be written as follows:
\begin{equation}
\begin{aligned}
\label{eq:grad2}
\frac{d{\boldsymbol{x}}}{dt} = R({\boldsymbol{x}})^{-1}\nabla \ln \mathit{p}({\boldsymbol{x}})=-R({\boldsymbol{x}})^{-1}\sum_{k=1}^K \omega_k(\boldsymbol{x})
\boldsymbol{\Sigma}_k^{-1}
(\boldsymbol{x}-\boldsymbol{\mu}_k),
\end{aligned}
\end{equation}
where $R(\cdot)$ is a {\em Riemannian metric} on $\mathcal{X}$ and
\begin{equation}
\begin{aligned}
\omega_k(\boldsymbol{x})=\frac{\pi_k(2\pi)^{-\frac{D}{2}}
|\boldsymbol{\Sigma}_k|^{-\frac{1}{2}}e^{-\frac{1}{2}(\boldsymbol{x}-\boldsymbol{\mu}_k)^T\boldsymbol{\Sigma}_k^{-1}
(\boldsymbol{x}-\boldsymbol{\mu}_k)}}{\sum_{k=1}^K \pi_k(2\pi)^{-\frac{D}{2}}|\boldsymbol{\Sigma}_k|^{-\frac{1}{2}}
e^{-\frac{1}{2}(\boldsymbol{x}-\boldsymbol{\mu}_k)^T\boldsymbol{\Sigma}_k^{-1}
(\boldsymbol{x}-\boldsymbol{\mu}_k)} }>0.
\end{aligned}
\end{equation}

We focus on MoG density because the most studies are on k-means clustering or MoGs, from which the MoG density can be estimated. 

\subsubsection{Differentially private mixture of Gaussians}

\cite{park2017dp} proposed a differentially private MoG (DPMoG) method by adding noise to parameters $\mu_k, \Sigma_k, \pi_k$ in the M-step of the EM algorithm. Unlike k-means, MoG parameters are calculated using all samples, requiring larger noise proportional to the number of clusters, which degrades DPMoG's performance. Therefore, we present DPMoG-hard, a modified version whose performance is less affected by the number of clusters. To enable parallel composition, we transform the responsibility obtained in the E step so each sample is assigned with a probability of 1 to a single cluster. This technique, known as hard EM \citep{same2007online}, is first applied to enhance DP. DPMoG-hard is an instantiation of \textbf{DPClustering} in Algorithm \ref{alg:proposed}.

\begin{algorithm}[tb]
  \caption{DPMoG-hard\textcolor{black}{: An instantiation of DPClustering}}
  \label{alg:DPMoGhard}
\begin{algorithmic}[1]
  \STATE {\bfseries Input:} \textcolor{black}{Input data $\{\boldsymbol{x}_n\}_{n=1}^N\in [-1,1]^{N\times D}$, privacy budget ($\epsilon, \delta$), number of clusters $K$, number of iterations $\tau$}
  \STATE {\bfseries Output:} Parameters $\{ \pi_k ^{\tau} \}_{k=1}^{K}, \{\boldsymbol{\mu}_k^{\tau}\}_{k=1}^{K}, \{\boldsymbol{\Sigma}_k^{\tau}\}_{k=1}^{K}$
  \STATE \textcolor{black}{Initialize parameters $\{ \pi_k ^0 \}_{k=1}^{K} \{\boldsymbol{\mu}_k^0\}_{k=1}^{K}, \{\boldsymbol{\Sigma}_k^0\}_{k=1}^{K}$}
  \STATE $r\leftarrow 1+3D+2D^2$;\quad$\sigma\leftarrow\sqrt{\frac{r\tau}{2}}\frac{\sqrt{\log{(1/\delta)}+\epsilon}+\sqrt{\log{(1/\delta)}}}{\epsilon}$ \textcolor{black}{\textit{ // Calculate from \ref{eq:alg_sigma}}}
  \STATE $C_k\leftarrow \emptyset, \quad k=1,\ldots,K$
  %\FOR {$k=1$ to $K$}
  %  \STATE $C_k\leftarrow \emptyset$
  %\ENDFOR
  \FOR {$t=1$ to $\tau$}
   %   \STATE \textit{//E-step}
      \FOR {$n=1$ to $N$}
        \FOR {$k=1$ to $K$}
          \STATE $\gamma_{nk}^{t}\leftarrow\pi_k^{t-1}\mathcal{N}(\boldsymbol{x}_n|\boldsymbol{\mu}_k^{t-1}, \boldsymbol{\Sigma}_k^{t-1})/\sum_{l=1}^K{\pi_l^{t-1}\mathcal{N}(\boldsymbol{x}_n|\boldsymbol{\mu}_l^{t-1}, \boldsymbol{\Sigma}_l^{t-1})}$
        \ENDFOR
        \STATE $k^*=\argmax_k{\gamma_{nk}^{t}}$;\quad$C_{k^*}\leftarrow C_{k^*}\cup\{n\}$\textcolor{black}{\textit{ // hard EM}}
      \ENDFOR
   %   \STATE \textit{//M-step}
      \FOR {$k=1$ to $K$}
        \STATE $N_k\leftarrow \vert C_k \vert + \mathcal{N}(0,\sigma^2)$;\quad$\pi_k^{t}\leftarrow N_k/N$
        \STATE $\boldsymbol{\mu}_k^{t}\leftarrow \frac{1}{N_k}(\sum_{n\in C_k}{\boldsymbol{x}_n}+\mathcal{N}(\boldsymbol{0},\sigma^2\boldsymbol{I}_D))$
        \STATE $\boldsymbol{\Sigma}_k^{t}\leftarrow \frac{1}{N_k}(\sum_{n\in C_k}{(\boldsymbol{x}_n-\boldsymbol{\mu}_k^{t})(\boldsymbol{x}_n-\boldsymbol{\mu}_k^{t})^T}+\mathsf{sym}(\mathcal{N}(\boldsymbol{0},\sigma^2\boldsymbol{I}_{D(D+1)/2})))$
      \ENDFOR
  \ENDFOR
\end{algorithmic}
\end{algorithm}

The detailed procedure of DPMoG-hard is presented in Algorithm \ref{alg:DPMoGhard}. For the mathematical formulation of this algorithm, refer to \cite{park2017dp}. Specifying a distribution (lines 9, 12-14) implies performing random sampling from that distribution.
$\mathsf{sym}$ transforms a $D(D+1)/2$-dimensional vector into a $D\times D$-dimensional symmetric matrix.   

Using the composition theorem, Algorithm \ref{alg:DPMoGhard} satisfies $(\epsilon,\delta)$-DP by setting $\sigma$ as specified in line 3. Due to DPMoG-hard's hard clustering nature, the parallel composition theorem ensures that the total privacy loss remains unaffected by the number of clusters, $K$.

\begin{proposition}
    Algorithm \ref{alg:DPMoGhard} is $(\epsilon, \delta)$-differentially private.
\end{proposition}
\begin{proof}
    The sensitivity of $|C_k|$ is 1, and the sensitivity of each coordinate of $\sum_{n\in C_k}{\boldsymbol{x}_n}$ is 2. $\sum_{n\in C_k}{\boldsymbol{x}_n\boldsymbol{x}_n^T}$ has $D(D+1)/2$ unique elements. While $D$ diagonal elements have sensitivity 1, the others have sensitivity 2. 
    
    Using remark \ref{remarks3} and \ref{remarks4}, it can be directly shown that for an iteration,
    \begin{itemize}[label={-},itemsep=-2pt,left=0pt]
    \item $N_k$ is $(1/2\sigma^2)$-zCDP,
    \item $\sum_{n\in C_k}{\boldsymbol{x}_n}+\mathcal{N}(\boldsymbol{0},\sigma^2\boldsymbol{I}_D)$ is $(2^2 D/2\sigma^2)$-zCDP,
    \item $\sum_{n\in C_k}{(\boldsymbol{x}_n-\boldsymbol{\mu}_k^{t})(\boldsymbol{x}_n-\boldsymbol{\mu}_k^{t})^T}+\mathsf{sym}(\mathcal{N}(\boldsymbol{0},\sigma^2\boldsymbol{I}_{D(D+1)/2}))$ is $(D/2\sigma^2)$-zCDP for diagonal elements and $(2^2 D(D-1)/4\sigma^2)$-zCDP for the others.
    \end{itemize}
    
    By remark \ref{remarks4}, $\tau$ iterations of the algorithm is $(1+3D+2D^2)\tau/2\sigma^2$-zCDP. According to remark \ref{remarks1}, Algorithm \ref{alg:DPMoGhard} is $(\epsilon,\delta)$-DP with 
    \textcolor{black}{
    \begin{equation}\label{eq:alg_sigma}
        \sigma=\sqrt{\frac{r\tau}{2}}\frac{\sqrt{\log{(1/\delta)}+\epsilon}+\sqrt{\log{(1/\delta)}}}{\epsilon},
    \end{equation}
where $r=1+3D+2D^2$.}
\end{proof}

\subsubsection{Differentially private k-means clustering}

There have been various studies on k-means clustering algorithms that satisfy DP \textcolor{black}{\citep{hu2023lightweight,epasto2024k,li2024gapbas}}. Despite differences in privacy and utility analysis, our approach can be applied to any of them since they all output centers and allocate samples to centers.
However, k-means clustering does not directly provide a density function. To address this, we build a Gaussian density function by calculating the covariance matrix of each cluster and using it as $\Sigma_k$ of each Gaussian distribution. As in the last line of Algorithm \ref{alg:DPMoGhard}, noise is when calculating the covariance matrices to ensure DP, causing a little additional privacy loss.

\subsection{Transition equilibrium vectors and the weighted graph} 
To generate clusters of arbitrary shapes using the system \eqref{eq:grad2}, we group similar basin cells based on mutual proximity or dissimilarity. Using adjacent SEVs and TEVs, we construct the weighted graph $G=(V,E)$ with a derived distance \citep{lee2009dynamic}: % 

\begin{itemize}
 \item[1.] The vertices $V$ of $G$ are the SEVs, $\boldsymbol{x}^0_{1},...,\boldsymbol{x}^0_{s}$, $i=1,...,s$ of
(\ref{eq:grad}).
 \item[2.] The edge $E$ of $G$ connect vertices of adjacent SEVs, $\boldsymbol{x}^0_{i}, \boldsymbol{x}^0_{j}$, with weights $d_E(
\boldsymbol{x}^0_{i}, \boldsymbol{x}^0_{j}):=\mathit{f}(\boldsymbol{x}^1_{ij})$, where $\boldsymbol{x}^1_{ij}$ is a TEV between them.
\end{itemize}

For $\mathcal{X}_a$, we then construct a sub-graph $G_a=(V_a,E_a)$ with:
\begin{itemize}
\item[1.] The vertices $V_a\subset V$ of $G_a$ consists of SEVs, $\boldsymbol{x}^0_i$, in $V$ with
$\mathit{f}(\boldsymbol{x}^0_i)<a$.
 \item[2.] The edge $E_a\subset E$ of $G_a$ consists of $(\boldsymbol{x}^0_{i},
\boldsymbol{x}^0_{j})\in E$ with $d_E(\boldsymbol{x}^0_{i}, \boldsymbol{x}^0_{j})<a$.
\end{itemize}

Since the distance metric is based on the density function, it describes the data distribution more effectively than Euclidean distance methods. 
It also efficiently partitions the entire data space \citep{lee2009dynamic}.
We generalize the results from \cite{lee2009dynamic} using the Riemannian metric in the context of MoG (Theorem \ref{thm:1}).

The next result shows the equivalence between the topological structures of $G_a$ and the connected components of $\mathcal{X}_a$.
\begin{proposition}\citep{lee2006dynamic}
With respect to the MoG (\ref{eq:mog2}), $\boldsymbol{x}^0_{i}$ and $\boldsymbol{x}^0_{j}$ are in the same connected component of the sub-graph
$G_a$ if, and only if, $\boldsymbol{x}^0_{i}$ and $\boldsymbol{x}^0_{j}$ are in the same cluster of the level
set $\mathcal{X}_a$, that is, each connected component of $G_a$ corresponds to a cluster of $\mathcal{X}_a$.
\label{pro:1}
\end{proposition}

Morse theory ensures a TEV exists between two adjacent SEVs for any Morse function. Since the distance $d_E$ is defined by the function value of the TEV, connecting SEVs based on this distance constructs a new sub-graph with a modified value of $a$. Thus, we can efficiently construct an optimal sub-graph for clustering data based on the partitioned data manifold.

The next result shows an inductive property of the system \eqref{eq:grad2}, which allows to assign a label to a new test point without retraining.
This naturally partitions the sample space by assigning points in different basin cells to their respective SEVs. 
%Although traditional density-based clustering methods assign the same label to points in the same connected component $C_i$, 
Conversely, traditional density-based clustering methods require retraining to label a new test point, wasting the privacy budget.

\begin{theorem}[Inductive Property]\label{thm:1}
Suppose that the Riemannian metric $R(\cdot)$ on $\mathcal{X}$ of the associated MoG gradient system (\ref{eq:grad2}) has a finite condition number. Then the whole data space is almost surely a pairwise disjoint union of the basin cells $\mathfrak{B}(\boldsymbol{x}^0_i)$ where ${\boldsymbol{x}^0_i}_{|i \in \{1, \ldots, s\}}$ are the SEVs of the system (\ref{eq:grad2}), i.e.
\begin{eqnarray} 
\mathcal{X}=\mathfrak{B}(\boldsymbol{x}^0_1)\dot\cup\cdots \dot\cup \mathfrak{B}(\boldsymbol{x}^0_s).
\end{eqnarray}
Here, an almost surely disjoint union $A\dot\cup B$ of two nonempty sets $A$ and $B$ means that $A\cap B$ has a Lebesgue measure zero.
\end{theorem}

\begin{proof}
Following the proof of Lasalle's invariance property theorem \citep{guckenheimer2013nonlinear,khalil2002nonlinear}, it can be easily shown that every bounded trajectory of the system (\ref{eq:grad2}) converges to one of the equilibrium vectors. Therefore it is enough to show that every trajectory is bounded.

Since $R({\boldsymbol{x}})^{-1}$ and $\Sigma_k^{-1}$ are positive definite, we can let $A_k(\boldsymbol{x})$ be the Cholesky factorization of the positive definite matrix $R({\boldsymbol{x}})^{-1}\Sigma_k^{-1}$ that satisfies $R({\boldsymbol{x}})^{-1}\Sigma_k^{-1}= A_k(\boldsymbol{x})^TA_k(\boldsymbol{x})$. Then $(\boldsymbol{x}^TR({\boldsymbol{x}})^{-1}\Sigma_k^{-1}\boldsymbol{x})= \|A_k(\boldsymbol{x})\boldsymbol{x}\|^2$. 
Since the condition number of $R({\boldsymbol{x}})$ is bounded, by the spectral theorem, there exist smooth eigenvalue functions $\lambda_{\mathrm{min}}^k(\boldsymbol{x}),\lambda_{\mathrm{max}}^k(\boldsymbol{x})>0$ and $\gamma>0$ such that
\footnotesize
\begin{eqnarray}
\|A_k(\boldsymbol{x})\|=\sqrt{\lambda_{\mathrm{max}}^k(\boldsymbol{x})},  \quad \forall k=1,...K,~\forall \boldsymbol{x}\in \mathcal{X},  \\ \kappa(A_k(\boldsymbol{x})):=
\left(\lambda_{\mathrm{max}}^k(\boldsymbol{x})/\lambda_{\mathrm{min}}^k(\boldsymbol{x})\right)^{1/2}\leq \gamma,  \quad \forall k=1,...K,~\forall \boldsymbol{x}\in \mathcal{X}
\end{eqnarray}
\normalsize
where $\|\cdot\|$ denotes the Euclidean- or $\ell _{2}$-norm.
Now let $V(\boldsymbol{x})=\frac12 \|\boldsymbol{x}\|^2$ and choose $\Upsilon>\gamma\max_{k}\|\boldsymbol{\mu}_k\|$. Then for any $L>\Upsilon$, and for all $\|\boldsymbol{x}\|=L$, we have
\footnotesize
\begin{eqnarray}
&\frac{\partial}{\partial t} V(\boldsymbol{x})&=\boldsymbol{x}^T\frac{d
\boldsymbol{x}}{dt}=
-\boldsymbol{x}^T\sum_k\omega_k(\boldsymbol{x})R({\boldsymbol{x}})^{-1}\Sigma_k^{-1}(\boldsymbol{x} - \boldsymbol{\mu}_k)\nonumber\\
&=&-\boldsymbol{x}^T\sum_k\omega_k(\boldsymbol{x})R({\boldsymbol{x}})^{-1}\Sigma_k^{-1}\boldsymbol{x} + \boldsymbol{x}^T\sum_k\omega_k(\boldsymbol{x})R({\boldsymbol{x}})^{-1}\Sigma_k^{-1}\boldsymbol{\mu}_k\nonumber\\
&=&-\sum_k\omega_k(\boldsymbol{x})\boldsymbol{x}^TA_k(\boldsymbol{x})^TA_k(\boldsymbol{x})\boldsymbol{x} + \sum_k\omega_k(\boldsymbol{x})\boldsymbol{x}^TA_k(\boldsymbol{x})^TA_k(\boldsymbol{x})\boldsymbol{\mu}_k\nonumber\\
&\leq&-\sum_k\omega_k(\boldsymbol{x})\|A_k(\boldsymbol{x})\boldsymbol{x}\|^2 + \sum_k\omega_k(\boldsymbol{x})\|A_k(\boldsymbol{x})\boldsymbol{x}\|\|A_k(\boldsymbol{x})\boldsymbol{\mu}_k\|\nonumber\\
&=&\sum_k\omega_k(\boldsymbol{x})\|A_k(\boldsymbol{x})\boldsymbol{x}\|(\|A_k(\boldsymbol{x})\boldsymbol{\mu}_k\| - \|A_k(\boldsymbol{x})\boldsymbol{x}\|)\nonumber\\
&\leq&\sum_k\omega_k(\boldsymbol{x})\|A_k(\boldsymbol{x})\boldsymbol{x}
\|(\sqrt{\lambda_{\mathrm{max}}^k(\boldsymbol{x})}\|\boldsymbol{\mu}_k\| - \sqrt{\lambda_{\mathrm{min}}^k(\boldsymbol{x})}\|\boldsymbol{x}\|)\nonumber
< 0
\end{eqnarray}
\normalsize

Therefore, for any $L>\Upsilon$, the trajectory starting from any point on $\|\boldsymbol{x}\|=L>\Upsilon$ always enters into
the bounded set $\|\boldsymbol{x}\|\leq L$, which implies that $\{\boldsymbol{x}(t):t\geq 0\}$ is bounded.% $\P$
\end{proof}

Algorithm \ref{alg:FindTPs} demonstrates the procedures for finding TEVs and the weighted graph $G$.
The input \( f \) is a density function obtained from \textbf{DPClustering}. For example, \( p \) can be computed using (\ref{eq:mog2}) from \(\{\pi_k\}, \{\boldsymbol{\mu}_k\} \), \( \{\boldsymbol{\Sigma}_k\} \) output by Algorithm \ref{alg:DPMoGhard}, with \( f = -\ln p \). \( m \) and \( \tau \), are hyperparameters.
We use the centers $\{\boldsymbol{\mu}_k\}$ instead of SEVs, as the density of a center is very close to the local minimum, indicating an SEV nearby. This approach omits additional steps to find SEVs, saving computation time compared to previous studies. 

To find TEVs, we present an efficient modification of the quadratic string search method in \cite{lee2007quadratic}. We constrain $\boldsymbol{m}^{t-1}$ from the last iteration to be the vertex of the quadratic function $y=-x^2+x$ in the transformed coordinates, where the two centers form the unit vector from on the x-axis. We split $\boldsymbol{u}$ into $m+1$ intervals and calculate the corresponding points on the quadratic function for each vertex of the interval. Then, we find the point with the smallest density and iterate \eqref{eq:grad2} from this point to compute $\boldsymbol{m}^t$ (line 9-12). Figure \ref{fig:findtev} illustrates this procedure. The remaining parts of the algorithm identifies the index-1 equilibrium vectors from the TEV candidates (line 16) and filters TEVs satisfying \eqref{eq:524} (lines 17-20).
The next result shows the dynamical invariance property of the MoG system \eqref{eq:grad2} that preserves DP, implying that Algorithm \ref{alg:FindTPs} does not incur any additional privacy loss.

\begin{theorem}[Preserving Differential Privacy]\label{thm:2}
Let $\mathcal{M}$ with $\mathrm{Range}(\mathcal{M})\subseteq\mathcal{X}$ be a randomized algorithm that is $(\epsilon,\delta)$-differentially private. Then under the same condition of Theorem \ref{thm:1}, the inductive dynamical processing (\ref{eq:grad2}) applied to $\mathcal{M}$ is $(\epsilon,\delta)$-differentially private, i.e. the solution trajectory $\boldsymbol{x}(t)$ of the MoG gradient system (\ref{eq:grad2}) with initial condition $\boldsymbol{x}\in \mathcal{M}$ is $(\epsilon,\delta)$-differentially private. 
\end{theorem}

\begin{algorithm}[!ht]
  \caption{TEVGraph}
  \label{alg:FindTPs}
\begin{algorithmic}[1]
  \STATE {\bfseries Input:} Centers $\{\boldsymbol{\mu}_k\}_{k=1}^{K}$, density function $f$, line search parameter $m$, number of iterations $\tau$%, step size $\eta$
  \STATE {\bfseries Output:} Weighted graph $G=(V,E)$%Transition points $\{\boldsymbol{t}\}$
  \STATE $V\leftarrow \{\boldsymbol{\mu}_k\}_{k=1}^{K}$;\quad$T\leftarrow \emptyset$
  %\STATE $tmp\leftarrow \emptyset$
  \FOR{$k=1$ to $K$}
    \FOR{$l=k+1$ to $K$}
        \STATE $i^*\leftarrow \argmax_{i \in \{1,\ldots,m\}}{f(\mu_k+\frac{i}{m+1}(\mu_l-\mu_k))}$
        \STATE $\boldsymbol{m}^0\leftarrow \mu_k+\frac{i^*}{m+1}(\mu_l-\mu_k)$;\quad$\boldsymbol{m}^0\leftarrow$ Integrate (\ref{eq:grad2}) from $\boldsymbol{m}^0$%\boldsymbol{m}^0-\eta \nabla f(\boldsymbol{m}^0)$
        %\\\textcolor{black}{\textit{ /* Iterate \eqref{eq:grad2} for a small step size */}}
        \FOR{$t=1$ to $\tau$} 
            \STATE $\boldsymbol{u}\leftarrow \boldsymbol{\mu}_l-\boldsymbol{\mu}_k$;\quad$\boldsymbol{v}\leftarrow \boldsymbol{m}^0-\boldsymbol{\mu}_k$
            \STATE $i^*\leftarrow \argmax_{i \in \{1,\ldots,m\}}{f(\mu_k+\frac{i}{m+1}\boldsymbol{u}}+(\frac{i}{m+1}-(\frac{i}{m+1})^2)(4\boldsymbol{v}-2\boldsymbol{u}))$
            \STATE $\boldsymbol{m}^t\leftarrow \mu_k+\frac{i^*}{m+1}\boldsymbol{u}+(\frac{i^*}{m+1}-(\frac{i^*}{m+1})^2)(4\boldsymbol{v}-2\boldsymbol{u})$
            \STATE $\boldsymbol{m}^t\leftarrow$ Integrate (\ref{eq:grad2}) from $\boldsymbol{m}^t$%\boldsymbol{m}^t-\eta \nabla f(\boldsymbol{m}^t)$
        \ENDFOR
        \STATE $\boldsymbol{t}_{tmp}\leftarrow$ Find the solution of $\nabla f(\boldsymbol{x})=0$ from $\boldsymbol{m}^{\tau}$
        \STATE$T\leftarrow T\cup\{\boldsymbol{t}_{tmp}\}$
    \ENDFOR
  \ENDFOR
    %\textcolor{black}{\textit{ /* Filter TEVs satisfying \eqref{eq:524} */}}
  \FOR {$\boldsymbol{t}\in T$}
    \IF {Hessian $\nabla^2f(\boldsymbol{t})$ has one negative eigenvalue} 
        \STATE $\boldsymbol{e}\leftarrow$ Eigenvector corresponding to the negative eigenvalue
        \STATE $\boldsymbol{x}_0\leftarrow\boldsymbol{t}+\varepsilon\boldsymbol{e}$;\quad$\boldsymbol{x}_1\leftarrow\boldsymbol{t}-\varepsilon\boldsymbol{e}$\quad for small $\varepsilon>0$
        \STATE $\boldsymbol{\mu}_0, \boldsymbol{\mu}_1\leftarrow$ Numerically integrate (\ref{eq:grad2}) from $\boldsymbol{x}_0, \boldsymbol{x}_1$
        \IF {$\boldsymbol{\mu}_0\neq\boldsymbol{\mu}_1$}
            \STATE $E\leftarrow E\cup\langle\boldsymbol{\mu}_0, \boldsymbol{\mu}_1, f(\boldsymbol{t})\rangle$
        \ENDIF
    \ENDIF
  \ENDFOR
\end{algorithmic}
\end{algorithm}

\begin{figure}[t]
\centering
\includegraphics[width=0.5\linewidth]{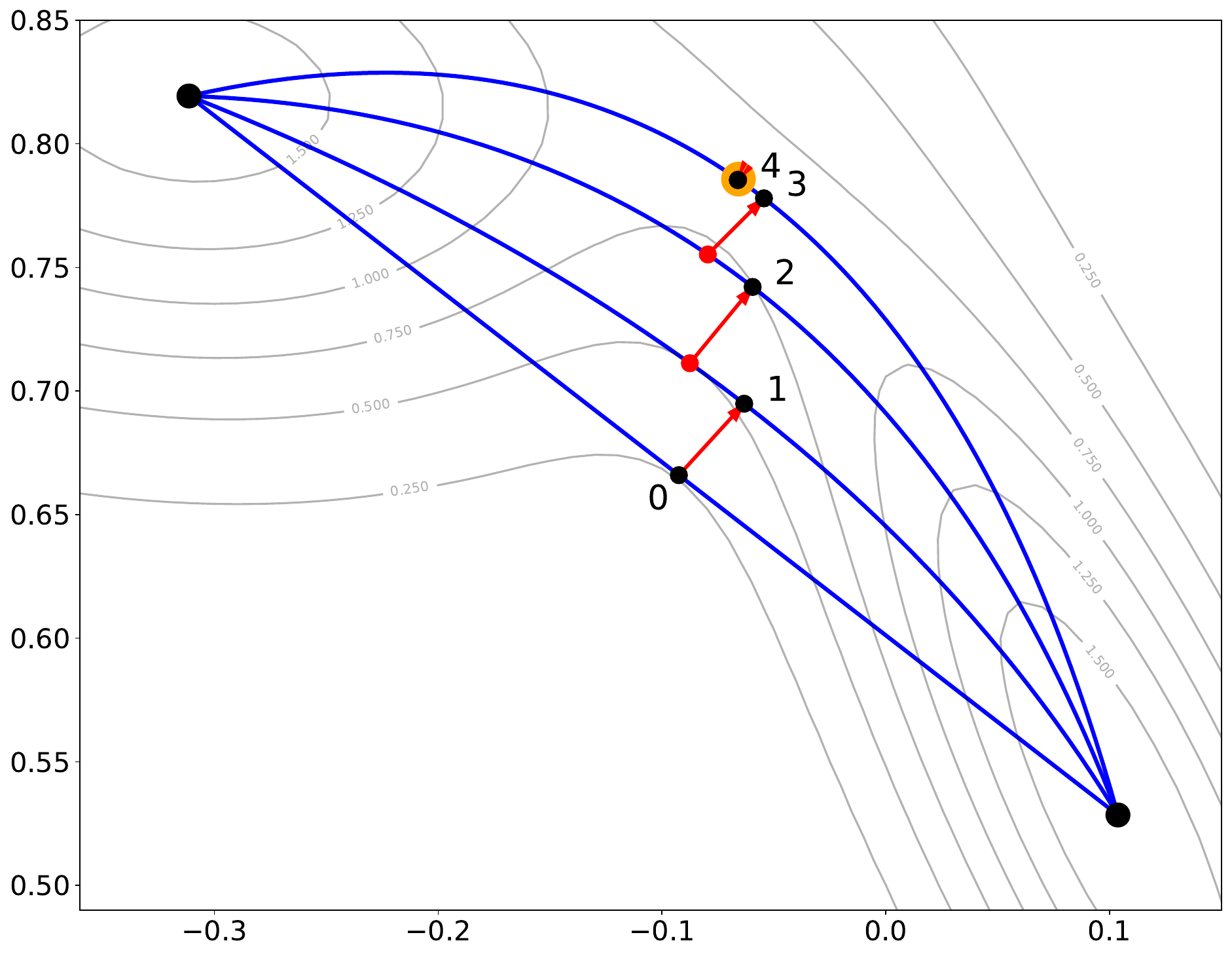}
\caption{Procedure to find TEVs. Big black points are MoG centers, and the big orange point is the TEV. Small black points with numbers represent $\boldsymbol{m}^t$ at step $t$, and small red points indicate minimum density points on the quadratic string. 
%The red arrow represents an iteration of the dynamical system (\ref{eq:grad2}).
}
\label{fig:findtev}
\end{figure}

\begin{proof}
Define the flow $\Phi_t:\mathcal{X}\to \mathcal{X}$ of the MoG gradient system (\ref{eq:grad2}) by $\Phi_t(\boldsymbol{x})=\boldsymbol{x}(t)$ with initial condition $\boldsymbol{x}(0)=\boldsymbol{x}$ for $t\in \mathbb{R}$. Then by the fundamental theorem of the flow defined by the system (\ref{eq:grad2}), we have $\Phi_{s+t}(\boldsymbol{x})=\Phi_{s}\circ\Phi_{t}(\boldsymbol{x})$ for all $\boldsymbol{x}\in \mathcal{X}$ and $s,t\in \mathbb{R}$. (See \cite{khalil2002nonlinear} for more details.) Now let a pair of neighboring databases $D_1, D_2$ with $\|D_1-D_2\|_1\leq 1$ be given.
Then for all $\boldsymbol{x}\in \mathcal{M}(D_1)$
\begin{eqnarray}
\boldsymbol{x}(t)=\Phi_{t}\circ\Phi_{0}(\boldsymbol{x})=\Phi_{t}\circ \boldsymbol{x}(0)
\end{eqnarray}
For any event $\mathcal{O}\subset \mathrm{Range}(\mathcal{M})\subseteq\mathcal{X}$ and any $t\in \mathbb{R}$, we let $\mathcal{W}_t=\{\boldsymbol{x} \in  \mathcal{X}: \boldsymbol{x}(t)\in \mathcal{O}\}$. Then we have
\begin{eqnarray}
\mathrm{Pr}\{\boldsymbol{x}(t)\in &\mathcal{O}&: \boldsymbol{x}\in \mathcal{M}(D_1)\}=\mathrm{Pr}[\mathcal{M}(D_1)\in \mathcal{W}_t]\nonumber\\
&\leq & e^{\epsilon}\mathrm{Pr}[\mathcal{M}(D_2)\in \mathcal{W}_t]+\delta \nonumber\\
& = & e^{\epsilon}\mathrm{Pr}\{\boldsymbol{x}(t)\in \mathcal{O}: \boldsymbol{x}\in \mathcal{M}(D_2)\}+\delta
\end{eqnarray}
Since this result works for any $t\in \mathbb{R}$, the inductive dynamical processing applied to $\mathcal{M}$ by system (\ref{eq:grad2}) is $(\epsilon,\delta)$-differentially private. 
\end{proof}

\subsection{Hierarchical merging of sub-clusters}\label{sec:hie_merg}

In the final step, sub-clusters are hierarchically merged until the desired number of clusters is achieved, using the weighted graph obtained in the previous step. 
Let the number of clusters, $K$, be arbitrarily given. 
The method begins with each center representing a singleton cluster. The detailed procedure is demonstrated in Algorithm \ref{alg:mergeSCs}. Denote these clusters $C_1=\{\boldsymbol{\mu}_1\},...,C_L=\{\boldsymbol{\mu}_L\}$. At each step, the two separate clusters containing two adjacent SEVs with the smallest edge weight distance are merged into a single cluster \textcolor{black}{(lines 10-11)}, producing one less cluster at the next higher level \textcolor{black}{(lines 12-13)}. This process continues until $K$ clusters are obtained from the initial $L$ clusters \textcolor{black}{(line 9)}. Since Algorithm \ref{alg:mergeSCs} operates only on the centers and TEVs, which preserve DP, it inherently maintains $(\epsilon, \delta)$-DP by the post-processing property.

Another feature of the MoG system is its agglomerative property, i.e. system \eqref{eq:grad2} can construct a hierarchy of clusters starting from the basin cells. This ensures that Algorithm \ref{alg:mergeSCs} can achieve any number of clusters.

\begin{theorem}[Agglomerative Property]\label{thm:connect}
Let $\boldsymbol{x}^0_i, i=1,...,s$ and $\boldsymbol{x}_j^1, j=1,...,t$ be the SEVs and the TEVs of the MoG system (\ref{eq:grad2}), respectively. Consider an agglomerative process such that each basin cell $\mathfrak{B}(\boldsymbol{x}^0_1)$ starts in its own cluster and a pair of clusters are merged when the two separate clusters contain adjacent basin cells with the least edge weight $\mathit{f}(\boldsymbol{x}_j^1)$. Then the merging occurs when we decrease the level value $a$ starting from $\max_i\mathit{f}(\boldsymbol{x}_i^0)$ until it hits the value in $\{\mathit{f}(\boldsymbol{x}_1^1),...,\mathit{f}(\boldsymbol{x}_t^1)\}$ and this process is terminated when we get one cluster, say $\mathcal{X}$, starting from $s$ clusters.
\end{theorem}

\begin{algorithm}[!ht]
  \caption{MergeCluster}
  \label{alg:mergeSCs}
\begin{algorithmic}[1]
  \STATE {\bfseries Input:} Weighted graph $G=(V,E)$, number of clusters $K$
  \STATE {\bfseries Output:} Clusters $\{C\}_{k=1}^K$
  \STATE $L\leftarrow \vert V\vert$, $C_1\leftarrow\{\boldsymbol{\mu}_1\},\ldots, C_L\leftarrow\{\boldsymbol{\mu}_L\}$ 
  \IF {$\langle\boldsymbol{\mu}_i,\boldsymbol{\mu}_j, f(\boldsymbol{t})\rangle\in E$}
    \STATE $d(C_i, C_j)\leftarrow f(\boldsymbol{t})$
  \ELSE
    \STATE $d(C_i, C_j)\leftarrow \infty$
  \ENDIF
  \FOR {$l=1$ to $L-K$}
    \STATE Find the smallest $d$ and the corresponding $C_a, C_b$
    \STATE $C_{L+l}\leftarrow C_a\cup C_b$\textcolor{black}{\textit{ // Merge two adjacent SEVs as in \ref{sec:hie_merg}}}
    \STATE $d(C_{L+1}, C_c) = \min\{d(C_a, C_c), d(C_b, C_c)\}$ for all remaining $C_c$s
    \STATE Remove $C_a$ and $C_b$
  \ENDFOR
\end{algorithmic}
\end{algorithm}

\begin{proof}
The first part of the proof comes from the Morse theory. For the second part of the proof, it is sufficient to show that $\eta>0$ exists, with $\mathcal{X}_r$ connected for all $0<r<\eta$.
From the Cholesky factorization of $\Sigma_k^{-1}$, we can let $\Sigma_k^{-1} = U_k^TU_k$ where $U_k$ is an upper triangle matrix with positive diagonal elements. By the singular value decomposition theorem, $0<\sigma_d^{(k)}\|\boldsymbol{x}\|\leq \|U_k\boldsymbol{x}\|\leq \sigma_1^{(k)}\|\boldsymbol{x}\|$ where $\sigma_1^{(k)}\geq \cdots \geq \sigma_d^{(k)}$ are the singular values for $k=1,...,K$. Let $a:=\min_k \sigma_d^{(k)}$, $b:=\max_k \sigma_1^{(k)}$. Choose $\zeta>\max\{\frac{b}{a},\gamma\}\cdot\max_{k}\|\boldsymbol{\mu}_k\|$, then for all $\|\boldsymbol{x}\|= L>\zeta$, we have
\begin{eqnarray}
\|U_k(\boldsymbol{x}-\boldsymbol{\mu}_k)\| \leq \|U_k \boldsymbol{x}\|+\|U_k\boldsymbol{\mu}_k\| < bL + \frac{a\sigma_1^{(k)}}{b}L \leq (b+a)L,
\end{eqnarray}
By the proof of Theorem~\ref{thm:1}, every trajectory starting from $\|\boldsymbol{x}\| = L$ for any $L>\zeta$ always enters into $\mathcal{S}_L:=\{\boldsymbol{x}:\|\boldsymbol{x}\|\leq L\}$, a connected and bounded set.
It is enough to show that there exists a $\eta>0$ such that $\mathcal{S}_L \subset \mathcal{X}_\eta$.

\begin{eqnarray}
p(\boldsymbol{x})& = &\sum_{k=1}^K \pi_k(2\pi)^{-\frac{d}{2}}|\boldsymbol{\Sigma}_k|^{-\frac{1}{2}} e^{-\frac{1}{2}(\boldsymbol{x}-\boldsymbol{\mu}_k)^T\boldsymbol{\Sigma}_k^{-1}
(\boldsymbol{x}-\boldsymbol{\mu}_k)} \nonumber\\
& > &\sum_{k=1}^K \pi_k(2\pi)^{-\frac{d}{2}}|U_k|e^{-\frac{1}{2}(a+b)^2L^2}\nonumber\\
%& \geq & \min_k|U_k| (2\pi)^{-\frac{d}{2}}e^{-\frac{1}{2}(a+b)^2L^2}\nonumber\\
& \geq &  (2\pi/a^2)^{-\frac{d}{2}}e^{-\frac{1}{2}(a+b)^2L^2}.
\end{eqnarray}
The last inequality follows from $\sum_{k=1}^K \pi_k = 1$ and $|U_k|=\prod_{i=1}^d \sigma_i^{(k)} \geq (\sigma_d^{(k)})^d\geq a^d$.
% \begin{eqnarray}
% |U_k|=\prod_{i=1}^d \sigma_i^{(k)} \geq (\sigma_d^{(k)})^d\geq a^d.
% \end{eqnarray}
When we choose
$\eta=(2\pi/a^2)^{-\frac{d}{2}}e^{-\frac{1}{2}(a+b)^2L^2}$, we have $\mathcal{S}_L \subset \mathcal{X}_r$ for all $0<r<\eta$.
Hence, by theorem \ref{thm:1}, for any point  $\boldsymbol{x}_0\in (\mathcal{X}_r\setminus \mathcal{S}_L )$, every trajectory starting from $\boldsymbol{x}_0$ should hit the boundary $L$ and enters into the region $\mathcal{S}_L $. This implies that the set $\mathcal{S}_L $ is a strong deformation retract of the level set $\mathcal{X}_r$, which shows that $\mathcal{X}_r$ is connected for all $r<\gamma$.
\end{proof}
This last step possesses a monotonicity property, meaning the dissimilarity between merged clusters increases monotonically with the level of the merger. As a result, a dendrogram can be plotted so that the height of each node is proportional to the inter-group dissimilarity value between its two daughters. The dendrogram is illustrated in Figure \ref{fig:dendrogram}.
\label{method}

\section{Experiments}
\label{experiment}

We evaluate the proposed method on various real-world datasets to verify that our method achieves better clustering results than existing methods.

\subsection{Datasets}
Six datasets were employed in this study: \textbf{Shuttle}, EMG physical action (\textbf{EMG}), MAGIC gamma telescope (\textbf{MAGIC}) from the UCI Machine Learing Repository \citep{Dua:2019}, Sloan digital sky survey (\textbf{Sloan}), Predicting pulsar star (\textbf{Pulsar}), and \textbf{USPS} from Kaggle\footnote{\href{https://www.kaggle.com/}{https://www.kaggle.com}} competitions.
To ensure DP, all variables were rescaled to the range of $[-1, 1]$.
The \textbf{EMG} dataset, originally with 10 classes, was modified by excluding similar actions. For the \textbf{USPS} dataset, dimensionality was reduced from 16$\times$16 to 3 using t-SNE.
%, as the efficacy of base methods tends to diminish with increasing dimensions. 
A detailed overview of the datasets is in Table \ref{tab:data}. 
% \footnote{\href{https://www.kaggle.com/}{https://www.kaggle.com}} 
% \footnote{\href{https://www.kaggle.com/lucidlenn/sloan-digital-sky-survey}{https://www.kaggle.com/lucidlenn/sloan-digital-sky-survey}}
%\footnote{\href{https://www.kaggle.com/colearninglounge/predicting-pulsar-starintermediate}{https://www.kaggle.com/colearninglounge/predicting-pulsar-starintermediate}}
%\footnote{\href{https://www.kaggle.com/datasets/bistaumanga/usps-dataset}{https://www.kaggle.com/datasets/bistaumanga/usps-dataset}} 

\begin{table}[t]
\footnotesize
\centering
\caption{Description of datasets.}
\label{tab:data}
\begin{tabular}{ccccc}
\hline
\textbf{Dataset} & \textbf{\begin{tabular}[c]{@{}c@{}}\# of \\ samples (N)\end{tabular}} & \textbf{\begin{tabular}[c]{@{}c@{}}\# of \\ variables (D)\end{tabular}} & \textbf{\begin{tabular}[c]{@{}c@{}}\# of \\ clusters (K)\end{tabular}} & \textbf{\begin{tabular}[c]{@{}c@{}}\# of \\ sub-clusters ($K_0$)\end{tabular}} \\ \hline
\textbf{Shuttle} & 43500 & 9 & 7 & 12 \\
\textbf{EMG} & 59130 & 8 & 6 & 10 \\
\textbf{MAGIC} & 19020 & 10 & 2 & 6 \\
\textbf{Sloan} & 10000 & 16 & 3 & 6 \\
\textbf{Pulsar} & 9273 & 8 & 2 & 6 \\
\textbf{USPS} & 7291 & 3 & 10 & 13 \\ \hline
\end{tabular}%
\normalsize
\end{table}

\subsection{Details on experiments}

We employed the following four baseline methodologies:
\begin{itemize}
    \item \textbf{DPMoG-hard}: The proposed differentially private MoG method aoutlined in Algorithm \ref{alg:DPMoGhard}.
    \item \textbf{DPLloyd} \citep{su2017differentially}\footnote{\href{https://github.com/DongSuIBM/PrivKmeans}{https://github.com/DongSuIBM/PrivKmeans}}: A differentially private adaptation of the Lloyd algorithm for k-means clustering. 
    %We implemented the version by \cite{su2017differentially} for enhanced execution. 
    \item \textbf{DPCube} \citep{balcan2017differentially}\footnote{\href{https://github.com/mouwenlong/dp-clustering-icml17}{https://github.com/mouwenlong/dp-clustering-icml17}}: A recent differentially private clustering approach designed for high-dimensional data.
    %, operating by iteratively generating small cubes.
    \item \textbf{DPTree}    \citep{chang2021locally}\footnote{\href{https://github.com/google/differential-privacy/tree/main/learning/clustering}{https://github.com/google/differential-privacy/tree/main/learning/clustering}}: \textcolor{black}{Generating a differentially private coreset algorithm.} %An algorithm for generating differentially private coresets.
    %After coreset generation, a nonprivate k-means algorithm is applied for private clustering. We denote this method as \textbf{DPTree} due to its tree-based coreset generation. Originally developed for local DP, the released code has been modified to adhere to global DP, aligning with the DP definition used in this study.
\end{itemize}

To the best of our knowledge, we incorporated all methods with available source code. Other studies on differentially private k-means clustering primarily focus on theoretical utility without providing empirical validations or publicly available code. Note that each baseline method can be treated as an instantiation of \textbf{DPClustering} in Algorithm \ref{alg:proposed}.

To demonstrate the applicability of the proposed method (Algorithm \ref{alg:proposed}) across various baseline methods, we conducted a comparative analysis of these methods before and after the application of the proposed framework. For clarity, we denoted the application of the proposed method by appending `\textbf{-Morse}' to each method's name. For instance, using \textbf{DPMoG-hard} as \textbf{DPClustering}, the proposed method was referred to as \textbf{DPMoG-hard-Morse}. We then compared the clustering metrics of \textbf{DPMoG-hard} with \textbf{DPMoG-hard-Morse}, extending this comparison to all baseline methods.

For k-means clustering methods, additional privacy loss is incurred when privately calculating covariance matrices. The allocation of the privacy budget to covariance calculation is flexible. For example, in \textbf{DPLloyd-Morse}, $\rho$ for zCDP is assigned proportionally to the number of elements requiring added noise.
% , as illustrated in Algorithm \ref{alg:DPLloyd_Morse}. 
Due to the increased complexity of algorithms for \textbf{DPCube} and \textbf{DPTree} compared to \textbf{DPLloyd}, line searches were conducted for these methods to determine the optimal distribution of the privacy budget.

The clustering metric used in the evaluation was the adjusted Rand index (ARI), which ranges from 0 and 1, with values closer to 1 indicating superior clustering. While ARI typically requires true labels, alternative clustering metrics that do not depend on true labels, such as silhouette scores, were found unsuitable for capturing complex cluster shapes.

%For \textbf{DPLloyd}, \textbf{DPCube}, \textbf{DPTree}, and their `\textbf{-Morse}' counterparts, 
The total privacy budget $\epsilon$ was set to \{10, 5, 2, 1, 0.5\}. For $\epsilon<0.5$, baseline methods exhibited low ARI scores for most datasets, contrasting with experiments using other metrics such as the k-means objective, where performance is retained for smaller $\epsilon$. For \textbf{DPMoG-hard} and \textbf{DPMoG-hard-Morse}, $\epsilon=0.5$ were excluded due to their suboptimal performance. 

All the experiments were performed on a machine with 48 threads of Intel Xeon CPU E5-2680 v3 @ 2.50GHz CPU. The experiments were conducted using Python 3.6.9, repeated five times, and the average results were obtained. In addition to the privacy budget $\epsilon$, several parameters were determined. The number of iterations $\tau_1$ and $\tau_2$ were set to 10 and 5, respectively, for all experiments. In Algorithm \ref{alg:FindTPs}, $m$ and $\varepsilon$ were set to 20 and 0.05, respectively.

\subsection{Effectiveness of dynamical processing} 

\subsubsection{Results on differentially private mixture of Gaussians}
The clustering outcomes of \textbf{DPMoG-hard} and \textbf{DPMoG-hard-Morse} for the six datasets are illustrated in Figure \ref{fig:res1}. The application of Morse theory in \textbf{DPMoG-hard-Morse} consistently yields higher ARI values compared to \textbf{DPMoG-hard}. Notably, in the case of \textbf{Pulsar} dataset, ARI exhibits a notable increase of up to 0.2, signifying a substantial enhancement. For the other datasets, ARI generally increases by 0.05 to 0.1. These results suggest that the proposed method effectively merges the generated sub-clusters to accurately represent clusters with arbitrary shapes.

\begin{figure*}[t]
\centering
\subfigure[\textbf{Shuttle}]{%
\includegraphics[width=0.317\linewidth]{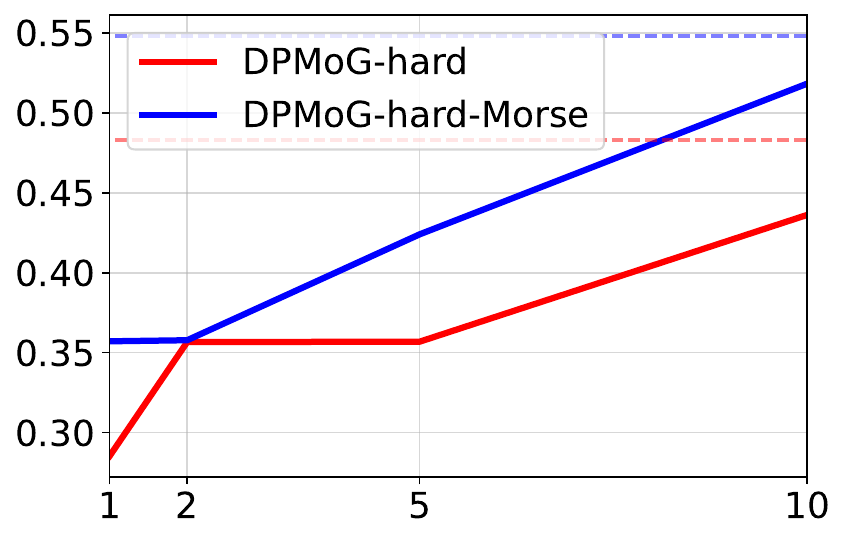}
\label{fig:shuttle-gmm}}
\subfigure[\textbf{EMG}]{%
\includegraphics[width=0.317\linewidth]{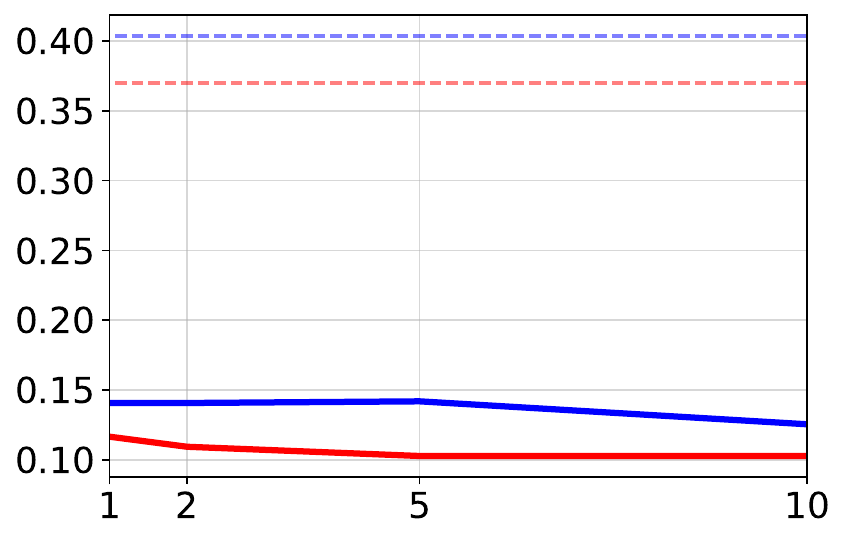}
\label{fig:emg-gmm}}
\subfigure[\textbf{MAGIC}]{%
\includegraphics[width=0.317\linewidth]{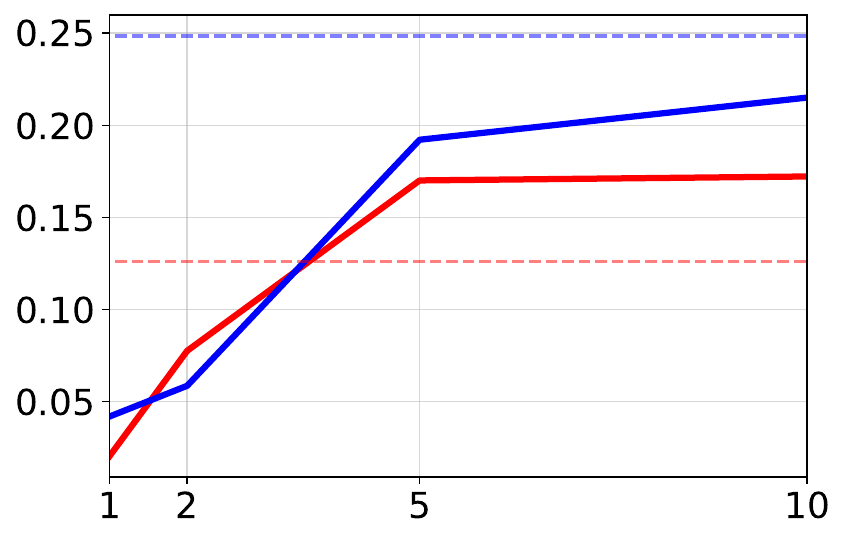}
\label{fig:magic-gmm}}
\subfigure[\textbf{Sloan}]{%
\includegraphics[width=0.317\linewidth]{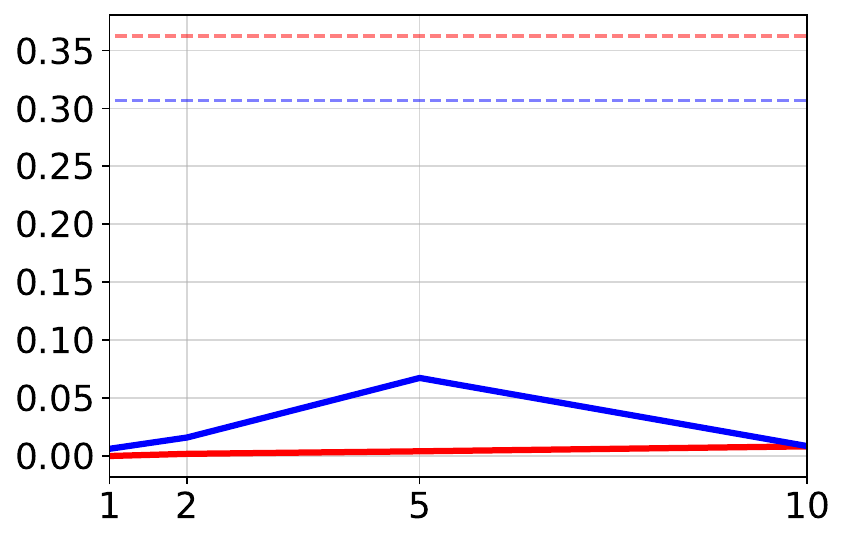}
\label{fig:sloan-gmm}}
\subfigure[\textbf{Pulsar}]{%
\includegraphics[width=0.317\linewidth]{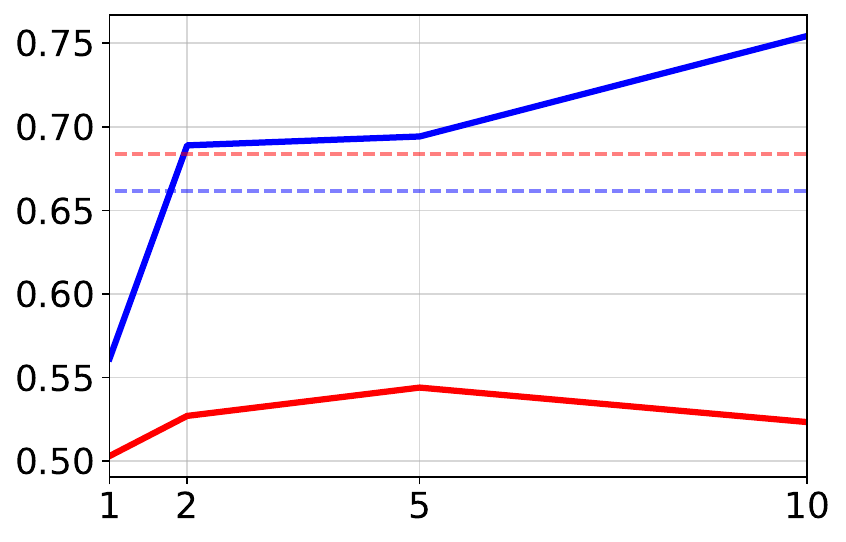}
\label{fig:pulsar-gmm}}
\subfigure[\textbf{USPS}]{%
\includegraphics[width=0.317\linewidth]{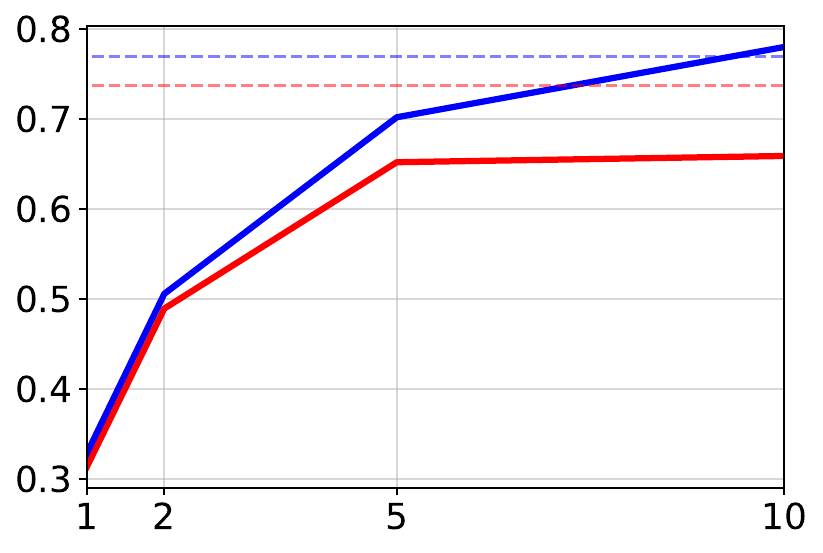}
\label{fig:usps-gmm}}
\caption{Clustering results of \textbf{DPMoG-hard} and \textbf{DPMoG-hard-Morse} for real-world datasets. The x-axis indicates the privacy budget  $\epsilon$, and the y-axis indicates the ARI score. 
The dotted lines indicate the performances of the non-private models.}
\label{fig:res1}
\end{figure*}

\begin{figure*}[t]
\centering
\subfigure[\textbf{Shuttle}]{%
\includegraphics[width=0.317\linewidth]{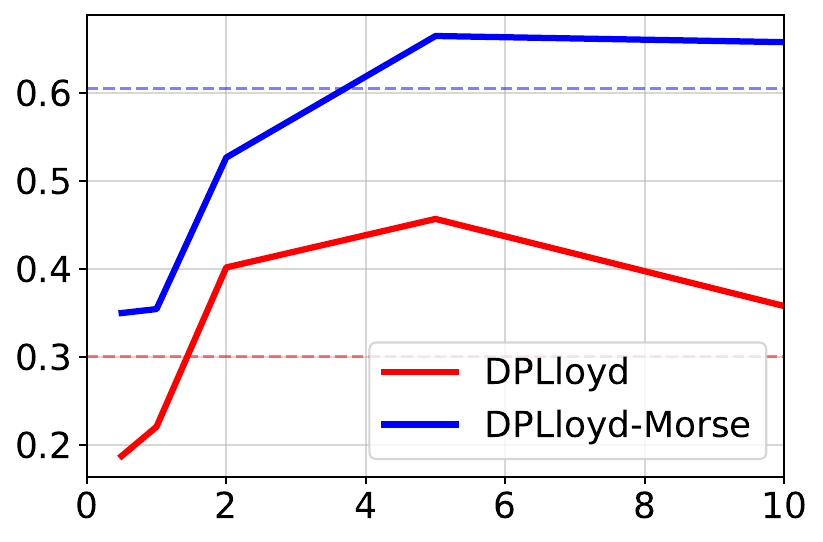}
\label{fig:shuttle-kmeans}}
\subfigure[\textbf{EMG}]{%
\includegraphics[width=0.317\linewidth]{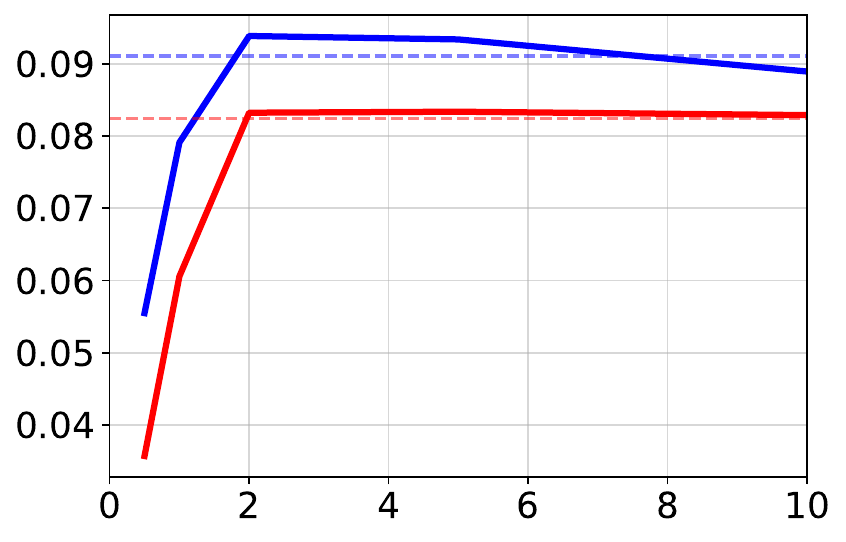}
\label{fig:emg-kmeans}}
\subfigure[\textbf{MAGIC}]{%
\includegraphics[width=0.317\linewidth]{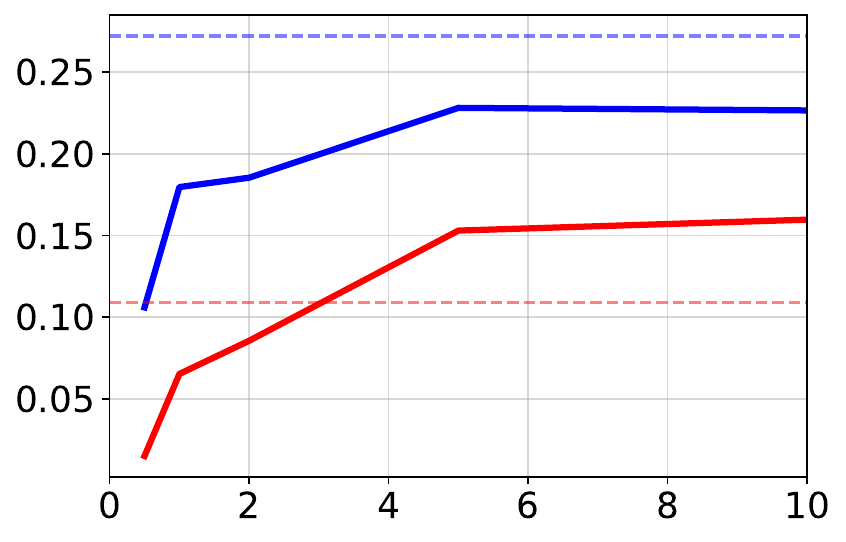}
\label{fig:magic-kmeans}}
\subfigure[\textbf{Sloan}]{%
\includegraphics[width=0.317\linewidth]{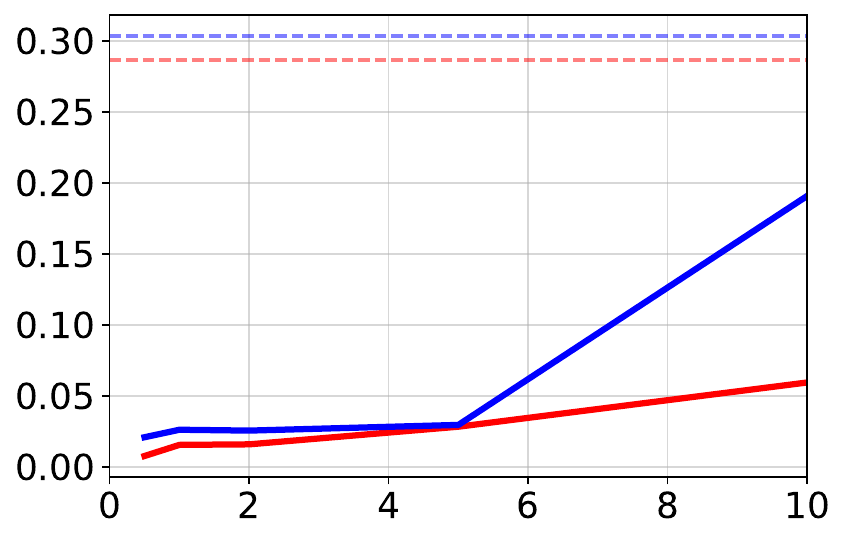}
\label{fig:sloan-kmeans}}
\subfigure[\textbf{Pulsar}]{%
\includegraphics[width=0.317\linewidth]{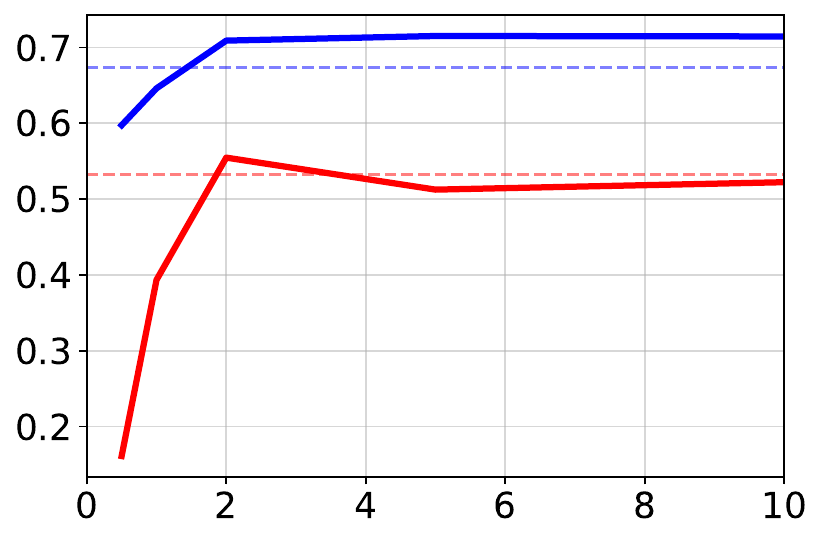}
\label{fig:pulsar-kmeans}}
\subfigure[\textbf{USPS}]{%
\includegraphics[width=0.317\linewidth]{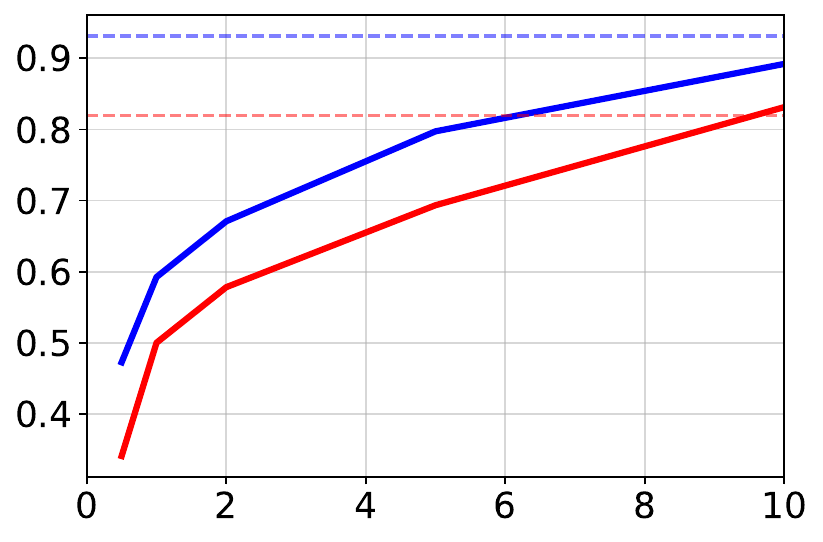}
\label{fig:usps-kmeans}}
\caption{Clustering results of \textbf{DPLloyd} and \textbf{DPLloyd-Morse} for real-world datasets. The x-axis indicates the privacy budget $\epsilon$, and the y-axis indicates the ARI score. The dotted lines indicate the performances of the non-private models.
}
\label{fig:res_kmeans}
\end{figure*}

\begin{figure*}[t]
\centering
\subfigure[\textbf{Shuttle}]{%
\includegraphics[width=0.317\linewidth]{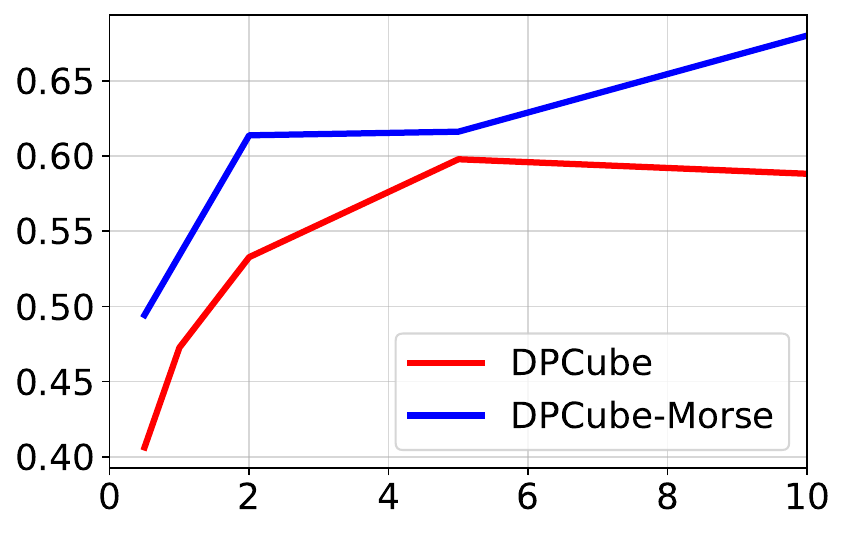}
\label{fig:shuttle-cube}}
\subfigure[\textbf{EMG}]{%
\includegraphics[width=0.317\linewidth]{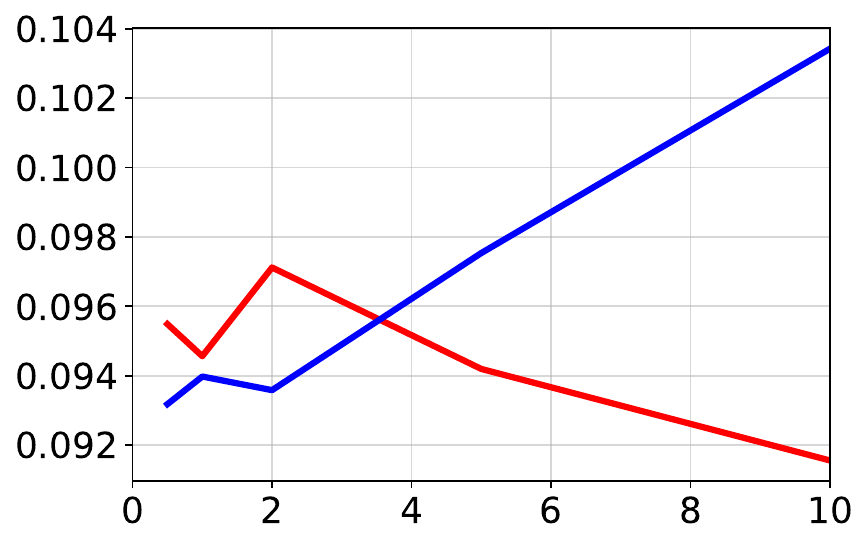}
\label{fig:emg-cube}}
\subfigure[\textbf{MAGIC}]{%
\includegraphics[width=0.317\linewidth]{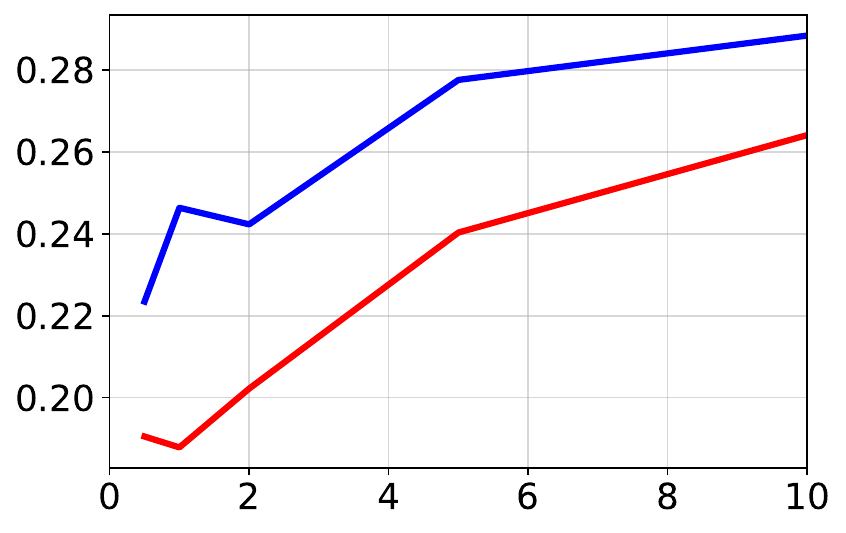}
\label{fig:magic-cube}}
\subfigure[\textbf{Sloan}]{%
\includegraphics[width=0.317\linewidth]{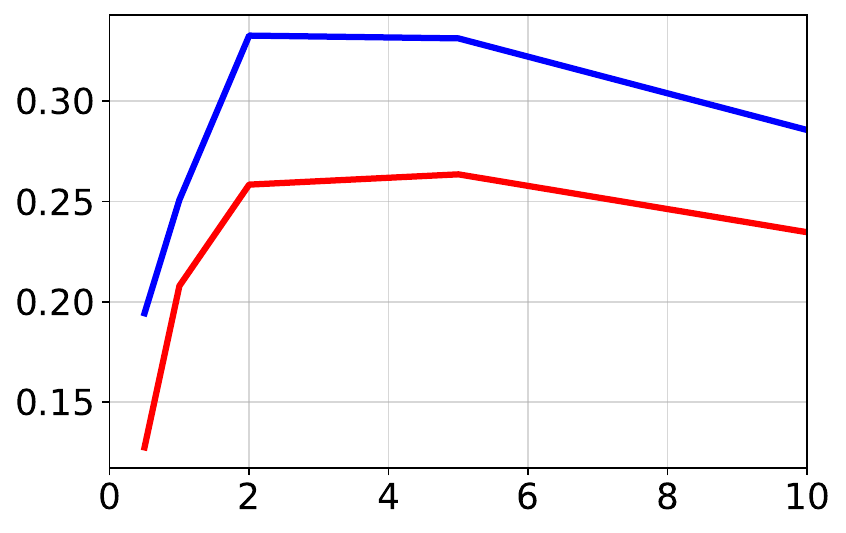}
\label{fig:sloan-cube}}
\subfigure[\textbf{Pulsar}]{%
\includegraphics[width=0.317\linewidth]{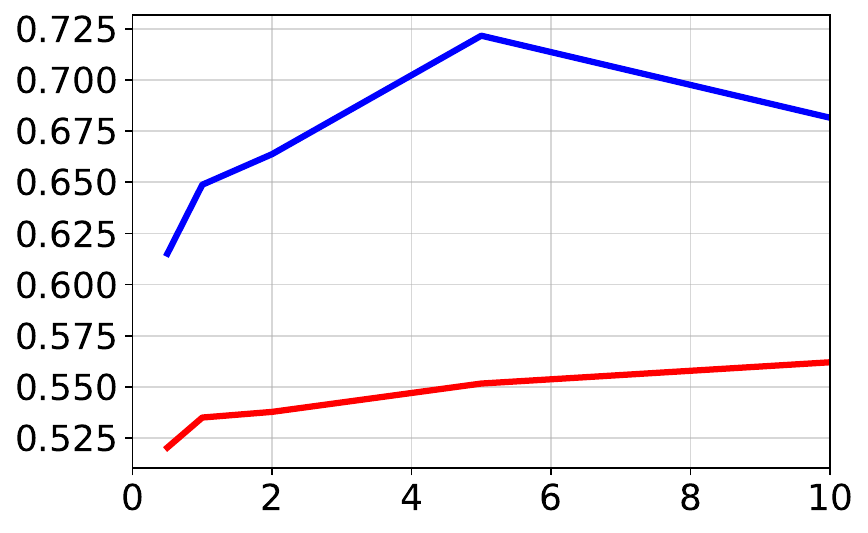}
\label{fig:pulsar-cube}}
\subfigure[\textbf{USPS}]{%
\includegraphics[width=0.317\linewidth]{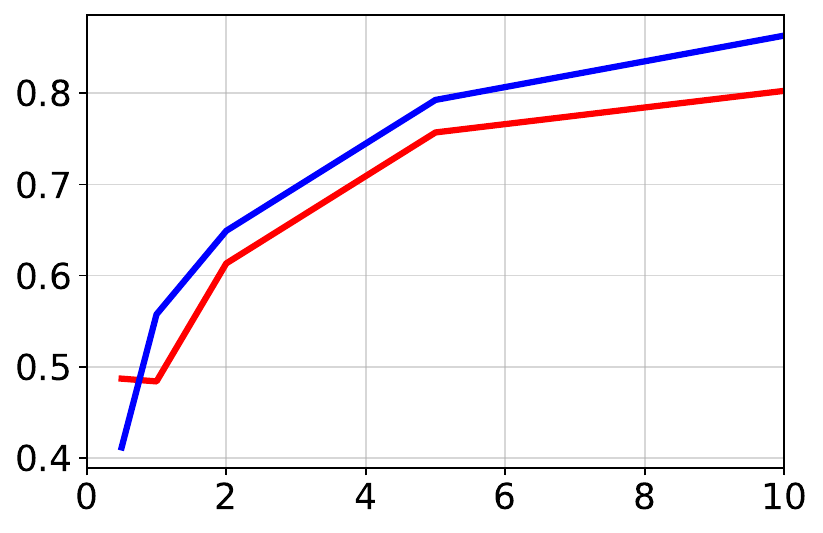}
\label{fig:usps-cube}}
\caption{Clustering results of \text{DPCube} and \textbf{DPCube-Morse} for real-world datasets. The x-axis indicates the privacy budget $\epsilon$, and the y-axis indicates the ARI score.
}
\label{fig:res_cube}
\end{figure*}

\begin{figure*}[t]
\centering
\subfigure[\textbf{Shuttle}]{%
\includegraphics[width=0.317\linewidth]{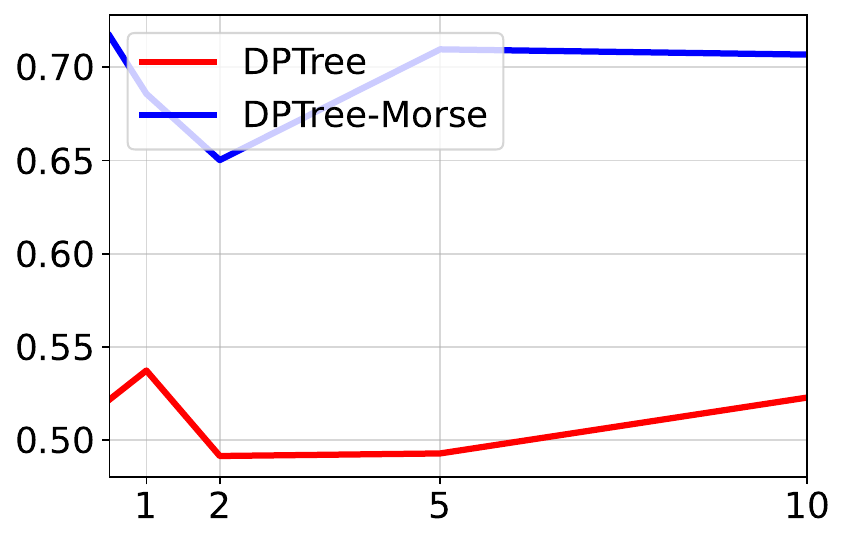}
\label{fig:shuttle-tree}}
\subfigure[\textbf{EMG}]{%
\includegraphics[width=0.317\linewidth]{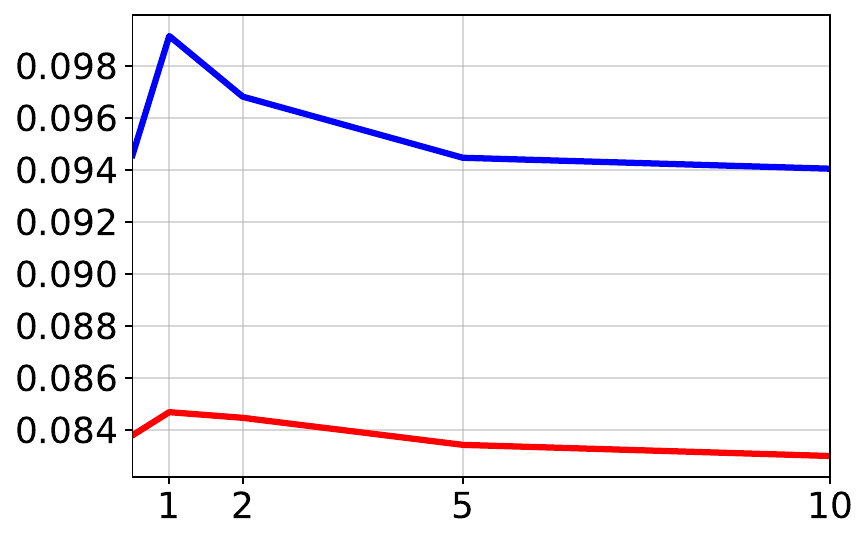}
\label{fig:emg-tree}}
\subfigure[\textbf{MAGIC}]{%
\includegraphics[width=0.317\linewidth]{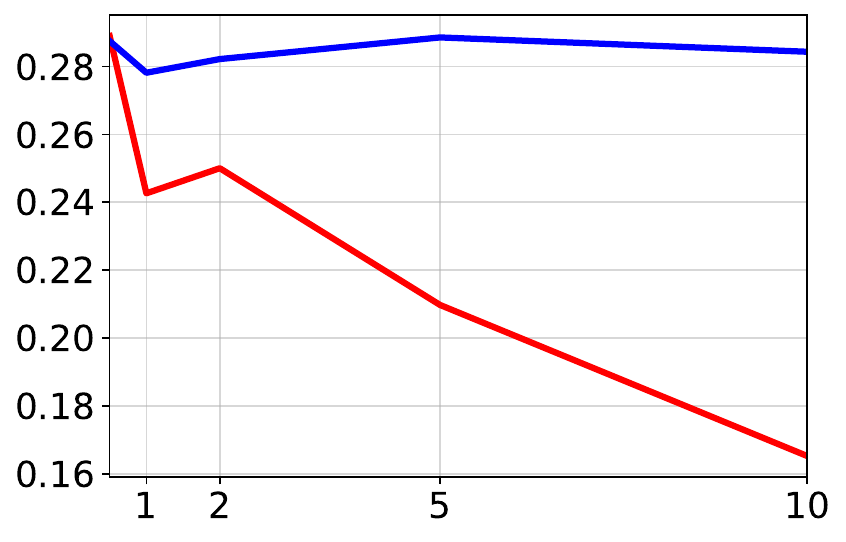}
\label{fig:magic-tree}}
\subfigure[\textbf{Sloan}]{%
\includegraphics[width=0.317\linewidth]{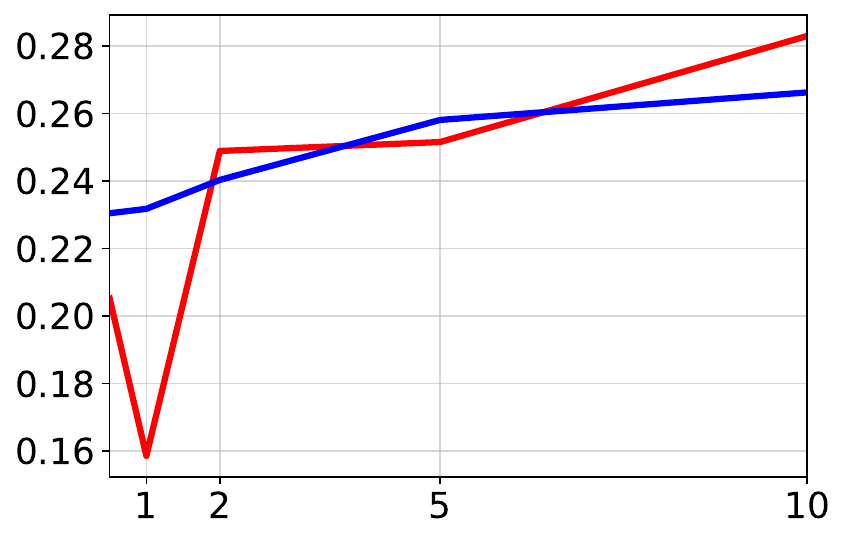}
\label{fig:sloan-tree}}
\subfigure[\textbf{Pulsar}]{%
\includegraphics[width=0.317\linewidth]{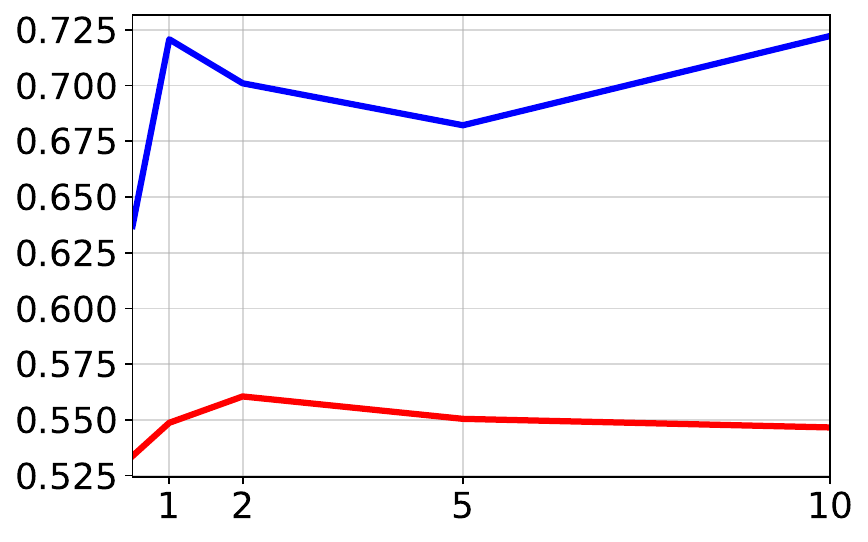}
\label{fig:pulsar-tree}}
\subfigure[\textbf{USPS}]{%
\includegraphics[width=0.317\linewidth]{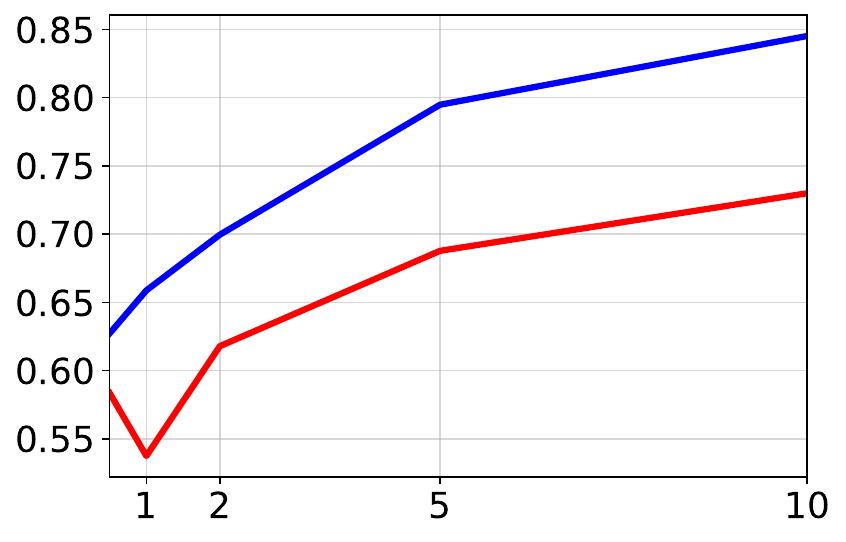}
\label{fig:usps-tree}}
\caption{Clustering results of \textbf{DPTree} and \textbf{DPTree-Morse} for real-world datasets. The x-axis indicates the privacy budget $\epsilon$, and the y-axis indicates the ARI score.
}
\label{fig:res_tree}
\end{figure*}

However, for most datasets, the disparity in ARI scores between \textbf{DPMoG-hard} and \textbf{DPMoG-hard-Morse} tends to diminish as $\epsilon$ decreases. Additionally, for the \textbf{EMG} and \textbf{Sloan} datasets, the improvement is relatively modest across all values of $\epsilon$. Both datasets consistently exhibit suboptimal performance compared to the non-private model even at $\epsilon=10$. This suggests a limitation of the proposed method, wherein the application of Morse theory does not significantly enhance performance when the performance of the baseline model is already severely compromised by noise to ensure DP. 

\subsubsection{Results on differentially private k-means}

Figure \ref{fig:res_kmeans} compares the clustering outcomes between \textbf{DPLloyd} and \textbf{DPLloyd-Morse} across the same six datasets. In nearly all instances, the ARI score of \textbf{DPLloyd-Morse} surpasses that of \textbf{DPLloyd}. Furthermore, akin to the observations for \textbf{DPMoG-hard}, the distinction in ARI scores between \textbf{DPLloyd} and \textbf{DPLloyd-Morse} diminishes with decreasing $\epsilon$. Specifically, for datasets such as \textbf{Shuttle}, \textbf{Magic}, \textbf{Pulsar}, and \textbf{USPS}, where ARI scores of \textbf{DPLloyd} exhibit moderate performance, the difference in ARI scores between \textbf{DPLloyd-Morse} and \textbf{DPLloyd} consistently proves to be significant, indicating a notable enhancement. Although the absolute improvement is marginal for \textbf{EMG} and \textbf{Sloan}, where \textbf{DPLloyd} demonstrates suboptimal ARI scores, the ratio between the ARI scores is found to be sufficiently substantial (at least a 12\% increase in most cases).

Figure \ref{fig:res_cube} illustrates the outcomes of \textbf{DPCube} and \textbf{DPCube-Morse}. Notably, \textbf{DPCube-Morse} consistently exhibits superior clustering performance compared to \textbf{DPCube} in regions characterized by higher privacy levels (larger values of $\epsilon$). In the case of \textbf{EMG} dataset, \textbf{DPCube} outperforms \textbf{DPCube-Morse} when $\epsilon\leq2$, though the difference is not significant due to the small scale of the y-axis. Compared to the results in Figure \ref{fig:res_kmeans}, two main differences are observed. Firstly, for the \textbf{Shuttle} and \textbf{Pulsar} datasets, \textbf{DPCube} exhibits superior clustering outcomes in high-privacy regions compared to \textbf{DPLloyd}, and this trend is preserved for \textbf{DPCube-Morse} relative to \textbf{DPLloyd-Morse}. This underscores the versatility of the proposed method across a spectrum of baseline models, from those with lower to higher performance levels. Secondly, \textbf{DPCube} consistently outperforms \textbf{DPLloyd} for the \textbf{MAGIC} and \textbf{Sloan} datasets. For \textbf{Sloan}, \textbf{DPCube} demonstrates moderate clustering results, while \textbf{DPLloyd}'s performance diminishes. This is because \textbf{DPCube} is effective for high-dimensional data. \textbf{DPCube-Morse} also shows improved performance, significantly surpassing \textbf{DPCube}.

Figure \ref{fig:res_tree} compares the clustering outcomes of \textbf{DPTree} and \textbf{DPTree-Morse}. The results parallel those of other methods, with ARI scores for \textbf{DPTree-Morse} consistently outperforming \textbf{DPTree} for all datasets except \textbf{Sloan}. In the case of \textbf{Sloan} dataset, \textbf{DPTree-Morse} exhibits superior performance in high-privacy regions, while both methods demonstrate comparable performance in low-privacy regions. An interesting observation is that, for the majority of datasets, the performance of \textbf{DPTree} appears to increase as $\epsilon$ decreases, a trend also evident in the performance of \textbf{DPTree-Morse}. This phenomenon may be attributed to the original design of \textbf{DPTree} as a model satisfying local DP, with the released code subsequently modified to conform to global DP, potentially leading to unintended outcomes.

In summary, the experimental results affirm that the proposed method consistently enhances the performance of various differentially private clustering methods. This enhancement is particularly pronounced when the baseline method exhibits a certain level of performance. The findings suggest the potential for the proposed method to be universally applicable, contingent upon the development of superior baseline models in future research.

\subsection{Effect of increasing sub-clusters} 
Theoretically, the proposed method can achieve the desired number of clusters by aggregating sub-clusters, regardless of the number of Gaussian sub-clusters created. However, due to the generation of different TEVs with varying numbers of centers, the final clustering outcomes exhibit variability. Therefore, an empirical investigation was conducted to assess the consistency of the final clustering performance across different numbers of sub-clusters.

\begin{table}[]
\caption{ARI scores of \textbf{DPMoG-hard-Morse} with different numbers of initial sub-clusters for \textbf{Pulsar} and \textbf{Magic} datasets.}
\label{tab:diffk}
\resizebox{\textwidth}{!}{%
\begin{tabular}{c|llll|llll}
\hline
\textbf{Dataset} & \multicolumn{4}{c|}{\textbf{Pulsar}} & \multicolumn{4}{c}{\textbf{Magic}} \\ \hline
\backslashbox{$\mathbf{\epsilon}$}
{$\mathbf{K_0}$} & \multicolumn{1}{c}{\textbf{6}} & \multicolumn{1}{c}{\textbf{10}} & \multicolumn{1}{c}{\textbf{15}} & \multicolumn{1}{c|}{\textbf{20}} & \multicolumn{1}{c}{\textbf{6}} & \multicolumn{1}{c}{\textbf{10}} & \multicolumn{1}{c}{\textbf{15}} & \multicolumn{1}{c}{\textbf{20}} \\ \hline
\textbf{10} & 0.754277 & 0.754838 & 0.766507 & 0.728404 & 0.259386 & 0.250220 & 0.248503 & 0.251199 \\
\textbf{5} & 0.694178 & 0.673453 & 0.659280 & 0.664384 & 0.217444 & 0.215024 & 0.209862 & 0.216947 \\
\textbf{2} & 0.688862 & 0.610148 & 0.596477 & 0.576486 & 0.083149 & 0.092159 & 0.087595 & 0.096408 \\
\textbf{1} & 0.561494 & 0.584608 & 0.577082 & 0.542264 & 0.040309 & 0.047554 & 0.040972 & 0.038205 \\ \hline
\end{tabular}%
}
\end{table}

Table \ref{tab:diffk} presents the clustering results of \textbf{DPMoG-hard-Morse} with varying numbers of initial sub-clusters for the \textbf{Pulsar} and \textbf{Magic} datasets. \textcolor{black}{For \textbf{Pulsar}, the observed trend indicates a tendency for clustering performance to decrease as the parameter $K_0$ increases. Two plausible interpretations arise from these results. Firstly, the results could be attributed to the increased number of TEVs. Alternatively, the increase in the number of} sub-clusters might reduce the number of samples belonging to each cluster, thereby amplifying the impact of noise. Conversely, the ARI scores for \textbf{Magic} does not exhibit a clear dependence on $K_0$. This may be attributed to the larger sample size of \textbf{Magic} compared to \textbf{Pulsar}. The sample complexity of the proposed method should be considered for future research.

\section{Discussion}
\label{discussion}
We proposed an effective differentially private clustering method that leverages Morse theory to express complex, nonconvex clusters without compromising privacy. Our theoretical results demonstrate that the dynamical processing associated with MoGs is completely stable and does not introduce any additional privacy loss to existing methods. Experimental results confirm that applying Morse theory enhances clustering utility for various baseline methods at the same privacy levels. Specifically, the proposed dynamical processing is versatile and can be integrated into any existing differentially private clustering approach capable of estimating an MoG density function.

Our experiments not only highlighted the strengths of the proposed method but also identified its limitations. Firstly, the proposed dynamical processing did not achieve significant performance improvements when the baseline method's performance was poor due to high privacy levels. Additionally, excessively increasing the number of sub-clusters can be detrimental to final performance when the sample size is insufficient. It is well known from numerous studies, including \cite{chaudhuri2011differentially}, that the performance of differentially private algorithms is better preserved with larger sample sizes.
Future research should investigate the impact of the number of sub-clusters on final performance and determine the optimal number of subclusters.

The proposed method can be extended to a diverse set of density functions. While the kernel density function is inherently not DP-friendly, ongoing research aims to overcome these limitations since kernel methods are essential for various unsupervised learning problems such as dimensionality reduction \citep{kim2014sentiment} and exemplar estimation \citep{son2018learning}. For example, \cite{park2023efficient} avoided using support vectors in the inference step by approximating the mapping of the kernel. Finally, the proposed method could be further extended by applying it to other hierarchical clustering methods \citep{xie2023scalable} using different distance measures.


\bibliographystyle{elsarticle-num}
%\bibliographystyle{IEEEtran}
\bibliography{references}

%\newpage
%\appendix
%\input{appendix}

%\begin{thebibliography}{00}



%\end{thebibliography}
\end{document}